%% file: final_submission/jmlr.tex
\documentclass[twoside,11pt]{article}

\usepackage[abbrvbib, preprint]{final_submission/jmlr2e}

\usepackage{amssymb}
\usepackage{amsmath}
\usepackage{url,nicefrac}
\usepackage[capitalize]{cleveref}  
\usepackage{autonum}
\usepackage[ruled,algo2e]{algorithm2e}
\usepackage{enumitem}
\usepackage[usestackEOL]{stackengine}
\usepackage{booktabs} 
\usepackage{graphicx} 
\usepackage{mathtools}
\RequirePackage{etex}
\usepackage{etoc} 

\input{cref_setup}
\input{raaz-macros}

\usepackage{lastpage}

\ShortHeadings{Two-Sided Nearest Neighbors}{Sadhukhan, Paul, Dwivedi}
\firstpageno{1}

\begin{document}

\title{Two-Sided Nearest Neighbors: An adaptive and minimax optimal procedure for matrix completion}

\author{\name Tathagata Sadhukhan$^*$ \email ts767@cornell.edu \\
       \addr Department of Statistics and Data Science\\
       Cornell University\\
       Ithaca, NY 14853-1198, USA
       \AND
       \name Manit Paul$^*$ \email paulman@wharton.upenn.edu \\
       \addr Department of Statistics and Data Science\\
       The Wharton School, University of Pennsylvania\\
       Philadelphia, PA 19104-1686, USA
       \AND
       \name Raaz Dwivedi \email rd597@cornell.edu \\
       \addr Operations Research and Information Engineering\\
       Cornell Tech, Cornell University\\
       New York, NY 10044, USA\\
       $*$ $=$ equal contribution}

\editor{}

\maketitle

\begin{abstract}
Nearest neighbor (NN) algorithms have been extensively used for missing data problems in recommender systems and sequential decision-making systems. Prior theoretical analysis has established favorable guarantees for NN when the underlying data is sufficiently smooth and the missingness probabilities are lower bounded. Here we analyze NN with non-smooth non-linear functions with vast amounts of missingness. In particular, we consider matrix completion settings where the entries of the underlying matrix follow a latent non-linear factor model, with the non-linearity belonging to a \Holder function class that is less smooth than Lipschitz. Our results establish following favorable properties for a suitable two-sided NN: (1) The mean squared error (MSE) of NN adapts to the smoothness of the non-linearity, (2) under certain regularity conditions, the NN error rate matches the rate obtained by an oracle equipped with the knowledge of both the row and column latent factors, and finally (3) NN's MSE is non-trivial for a wide range of settings even when several matrix entries might be missing deterministically. We support our theoretical findings via extensive numerical simulations and a case study with data from a mobile health study, HeartSteps.
\end{abstract}

\section{Introduction}
\label{sec:intro}
Latent factor models are ubiquitous in recommendation systems, panel data settings, sequential decision-making problems, and in various other scenarios. Matrix completion is a crucial problem in this context. Suppose $\Theta = ((\theta_{i,j})) \in \mathbb{R}^{n \times m}$ denotes the matrix of ground truths and $X \in \mathbb{R}^{n \times m}$ denotes the observed matrix. Let $A_{i, j}$ be the indicator variable denoting whether the $(i, j)$-th element of the matrix has been observed or not. We have the following model,
\begin{align}
\label{eq:main_model}
    X_{i, j} = \begin{cases}
        \theta_{i,j} + \epsilon_{i, j} \quad & \mbox{if  } A_{i, j} = 1 ,\\
        * \quad & \mbox{if  } A_{i, j} = 0.
    \end{cases}
\end{align}
Here $\epsilon_{i,j}$ is mean zero noise. The primary objective of matrix completion problem is to estimate the ground truths $\theta_{i,j}$ for both the missing as well as non-missing entries. Without any assumption on the matrix $\Theta$ this is a very difficult problem as there are a large number ($nm$) of unknown parameters as opposed to number of observations in this problem. To make this problem feasible it is generally assumed that the matrix $\Theta$ has an implicit low dimensional structure i.e.\ there are row latent factors $u_1,\cdots,u_n \in \mathbb{R}^{d_1}$, column latent factors $v_1,\cdots,v_m \in \mathbb{R}^{d_2}$, and a latent function $f$ such that the following holds, 
\begin{align}
    \theta_{i,j} = f(u_i, v_j) \quad \forall \quad (i, j)\in [n] \times [m].
\end{align}
A very popular choice of a bilinear latent function $f$ is $f(u,v) = \langle u, v \rangle$. In this case, we have the decomposition $\Theta = UV^T$ where $U,V$ are the matrices containing the row and column latent factors respectively. More generally, there are a large number of works (refer to \cite{xu2013speedup}, \cite{jain2013provable}, \cite{zhong2015efficient}, \cite{chiang2015matrix}, \cite{lu2016sparse}, \cite{guo2017convex}, \cite{eftekhari2018weighted}, \cite{ghassemi2018global}, \cite{chiang2018using}, \cite{arkhangelsky2019synthetic}, \cite{bertsimas2020fast}, \cite{agarwal2020synthetic}, \cite{agarwal2021causal}, \cite{burkina2021inductive}) on matrix completion which assume that the ground truth matrix $\Theta$ can be decomposed as $U\Sigma V^T$ where $U, V$ are the covariance matrices comprising of row and column latent factors respectively.
However the setting when $f$ is unknown and non-linear which is the main focus of this work, has received relatively less attention in the literature. Nearest neighbor (NN) algorithms have been observed to perform well in this set-up. There are many variants of NN algorithm in literature which people have tried in this context. NN algorithms approximate the $L_2$ distance between the rows and columns in the latent functional space and use those estimated distances to obtain a fixed-radius NN estimator. One of the most prominent works in this domain is that of \cite{dwivedi2022counterfactual} who analyses the performance of row (user)-nearest neighbor with the objective of performing counterfactual inference in sequential experiments under the assumption that non-linear latent function $f$ is Lipschitz in both the coordinates $u$ and $v$. Another work by \cite{dwivedi2022doubly} also ventures into this regime assuming the latent function $f$ is a Lipschitz function satisfying certain convexity conditions and studies the performance of a doubly-robust nearest neighbor. Apart from this, \cite{yu2022nonparametric} works with an unknown \Holder-continuous latent function $f$  and has introduced a novel algorithm (a variant of the vanilla two-sided NN) which attains the minimax optimal non-parametric rate in a moderate regime assuming the knowledge of column latent factors.

We study the performance of the two-sided NN (TS-NN) method under the assumption that the latent function $f$ is \Holder-smooth, none of the row or column latent factors are observed, and the entries of the matrix are missing not at random (MNAR). This non-parametric setting considers a much more general model class than the low rank bilinear class of functions. The assumption of the non-parametric model class, such as the one we analyse in our work, has been previously studied in \cite{song2016blind}, \cite{li2019nearest}, \cite{dwivedi2022doubly}, and \cite{yu2022nonparametric}. Similar models have been previously widely studied in graphon estimation literature (with binary observations and symmetric matrix). \cite{gao2015rate}, \cite{gao2016optimal}, \cite{klopp2017oracle}, and \cite{xu2018rates} are some of the relevant references. Moreover, the MNAR regime is much closer to reality as compared to the missing completely at random (MCAR) regime. For instance in movie recommendation system, a user who does not like the action genre is very less likely to see movies with heavy action. \cite{schnabelfwang16}, \cite{ma2019missing}, \cite{zhu2019high}, \cite{sportisse2020imputation}, \cite{sportisse2020estimation_PCA}, \cite{wang2020causal}, \cite{yang2021tenips}, \cite{bhattacharya2021matrix}, and \cite{agarwal2021causal} are some of the several works in the literature which have considered the MNAR regime. Thus our work serves as a unification of these two domains of research.




\paragraph{Our contributions} 
 The main finding of our work is that not only does the two-sided nearest neighbor method adapt to the smoothness of the latent function $f$ but it also attains the minimax optimal non-parametric rate in an intermediate regime. In other words, even without the prior knowledge of the latent factors, the performance of the TS-NN is as good as the oracle algorithm which has access to all the latent factors. Our analysis also shows that the TS-NN algorithm is robust to the missingness pattern of the matrix and can yield minimax optimal rate even when some entries of the matrix are missing deterministically. Our work contributes to the growing literature on understanding properties/robustness of nearest neighbor methods and handling missing data problems with non-smooth data and deterministic missingness.

 \paragraph{Organization} We start with the description of the model and various underlying assumptions in \Cref{sec:assump}. In \Cref{sec:algorithm}, we review the two-sided nearest neighbor algorithm. We discuss the theoretical guarantees of the performance of the method in \Cref{sec:theo_guarantee} and support the theoretical results with extensive simulation studies and real data analysis in \Cref{sec:experiments}.

 \paragraph{Notations} We denote the set $\{1, \cdots, n\}$ by $[n]$. We use $a_n = O(b_n)$ or $a_n \ll b_n$ to imply that there exists a constant $c > 0$ such that $ a_n \leq c b_n$. We use $a_n = o(b_n)$ to imply that $a_n/b_n \rightarrow 0$ as $n \rightarrow \infty$.  We use $a_n = \Omega(b_n)$ to mean $b_n = O(a_n)$ and $a_n = \omega(b_n)$ to mean $b_n = o(a_n)$. The notation $a_n = \Theta(b_n)$ is used when both $a_n = O(b_n)$ and $a_n = \Omega(b_n)$ hold true. We use $\mathcal{U}$ to denote the set of row latent factors and $\mathcal{V}$ to denote the set of column latent factors. In our results, we use $c$ to denote universal constant (independent of $m, n$, model parameters), that might take a different value in every appearance.

 \section{Problem set-up}
\label{sec:assump}
We have the data matrix $X \in \mathbb{R}^{n \times m}$ coming from the data-generating model \eqref{eq:main_model}. The objective is to estimate the ground truth matrix $\Theta$ given the data matrix. We make the following assumptions regarding the data-generating mechanism and the structure of the ground truths $\theta_{i, j}$ for all $(i , j) \in [n] \times [m]$. 
\begin{assumption}[Non-linear factor model]
\label{asump_low_rank}
Conditioned on the latent factors $u_1,\cdots,u_n$ and $v_1,\cdots,v_m$ the ground truth has the following low-rank representation,
\begin{align}
    \theta_{i,j} = f(u_i, v_j)  \quad \forall \quad (i, j)\in [n] \times [m],
\end{align}
where $f$ is $(\lambda, L)$ \Holder function for $\lambda \in (0,1]$ i.e.\ for $x, x' \in \mathrm{Domain}(f)$,
\begin{align}
    |f(x) - f(x') | \leq L||x - x'||_{\infty}^{\lambda}.
\end{align}
\end{assumption}
\Cref{asump_low_rank} describes the non-parametric model where the ground truth matrix $\Theta$ is described in terms of the latent factors $\mathcal{U, V}$ using the latent function $f$. This allows for a potentially non-linear relationship between the user and time latent factors. 
\begin{assumption}[Sub-gaussian noise]
\label{asump_bounded_noise}The noise terms $\{\epsilon_{i, j}\}$ are independent of each other, the latent factors, and the missingness indicators $\{A_{i, j}\}$. Moreover $\{\epsilon_{i, j}\}$ are sub-gaussian random variables with $\mathbb{E}[\epsilon_{i, j}] = 0$, $\mathrm{Var}(\epsilon_{i, j}) = \sigma^2$.  
\end{assumption}

\begin{assumption}[Row and column latent factors] 
\label{asump_row_col}The row latent factors $u_1,\cdots,u_n$ are sampled independently from $\mathrm{Uniform}[0,1]^{d_1}$. The column latent factors $v_1,\cdots,v_m$ are sampled independently from $\mathrm{Uniform}[0,1]^{d_2}$. 
\end{assumption}
\Cref{asump_row_col} is made for the ease of presentation. The assumption of sampling the row and column latent factors from the unit hypercubes can be easily relaxed to any compact set in $d_1$ and $d_2$ dimensions (respectively). Moreover the entire analysis can be done for any arbitrary sampling distribution if we replace the tail bounds of the uniform distribution with that of the arbitrary distribution. 

We note that if both \Cref{asump_low_rank} and \Cref{asump_row_col} hold, then the latent function $f$ is bounded as it is well known that a \Holder-continuous function on a compact domain is always bounded. In all the results that we discuss in this paper both these assumptions are required to hold true and thus we assume that $|f(x)| \leq M$ for all $x \in \mbox{Domain}(f)$. 

\section{Algorithm}
\label{sec:algorithm}
We now describe the two-sided nearest neighbor (TS-NN) algorithm in this section. To approximate the $L_2$ distance in the latent functional space we consider the following oracle distance:
\begin{align}
  d_{row}^2(i , i') &= \frac{1}{m}\sum_{j \in [m]} (f(u_i, v_j) - f(u_{i'} , v_j))^2, \\
   d_{col}^2(j, j') &= \frac{1}{n}\sum_{i \in [n]} (f(u_i, v_j) - f(u_i, v_{j'}))^2,
\end{align}
for all $i, i' \in [n]$ and for all $j, j' \in [m]$. Here $d_{row}^2(i , i')$ serves as a proxy for the distance between the row latent factors $u_i, u_{i'}$. Similarly $d_{col}^2(j, j')$ serves as proxy for the distance between the column latent factors $v_j, v_{j'}$. However since the the latent function $f$ is unknown it is not possible to exactly compute these distances between the rows and the columns. Therefore we use the observed entries of the matrix $X$ to approximate the distances $d_{row}^2(i , i')$ and $d_{col}^2(j,j')$ via the following data-driven analogues:
\begin{align}
    \widehat d_{row}^2(i , i') &=  \frac{ \sum_{j \in [m]} (X_{i,j} - X_{i', j})^2 A_{i, j}A_{i',j}}{\sum_{j \in [m]} A_{i, j}A_{i', j}} - 2\sigma^2,\\
    \widehat d_{col}^2(j, j') &=  \frac{ \sum_{i \in [n]} (X_{i,j} - X_{i,j'})^2 A_{i, j}A_{i, j'}}{\sum_{i \in [n]} A_{i, j}A_{i, j'}} - 2\sigma^2.
\end{align}
We note that the variance of the noise terms, $\sigma^2$, in the definition of $ \widehat d_{row}^2(i , i'),  \widehat d_{col}^2(j,j')$, is without loss of generality and only to simplify the algebraic expressions henceforth. In practice, this term is not used and one can verify below that the algorithm is unaffected if we remove $2\sigma^2$ from the display above and replace $\eta_{row}^2, \eta_{col}^2$ with $\eta_{row}^2-2\sigma^2, \eta_{col}^2-2\sigma^2$.  

Given the distances above, the two-sided nearest neighbor algorithm with tuning parameters $\mbi{\eta} = \{\eta_{row}, \eta_{col}\}$ (TS-NN($\mbi{\eta}$)) consists of the following steps:
\begin{enumerate}
    \item Compute the pairwise row and column distance estimates $\widehat d_{row}^2(i , i')$ and $\widehat d_{col}^2(j, j')$ for all $i, i' \in [n]$ and for all $j, j' \in [m]$ and use those to construct the following neighborhoods, 
    \begin{align}
    \label{eq:defn_nrowcol}
         \cn_{row}(i) &=\{i' \in [n]: \what{d}_{row}^2(i, i' )\leq \eta^2_{row}\},\\
        \cn_{col}(j) &=\{j' \in [m]: \what{d}_{col}^2(j, j')\leq \eta^2_{col}\}.
    \end{align}
    \item  Average the outcomes across the the two sets of neighbors:
    \begin{align}
        \widehat \theta_{i, j} = \frac{\sum_{i' \in \cn_{row}(i); j' \in \cn_{col}(j)}X_{i',j'}A_{i', j'}}{|\deno|}, 
    \end{align}
where $\deno = \{(i', j')| i' \in \cn_{row}(i),\mbox{ }j' \in \cn_{col}(j),\mbox{ } A_{i', j'} = 1 \} $.
\end{enumerate}

\section{Theoretical guarantees}
\label{sec:theo_guarantee}
In this section, we present our main results that characterize the performance of the two-sided nearest neighbor algorithm under the assumptions discussed in \Cref{sec:assump}. We discuss non-asymptotic guarantees at both population level as well as at row$\times$column level. We further complement these non-asymptotic guarantees with a result of asymptotic normality of $\widehat \theta_{i,j}$. 
\subsection{Non-asymptotic guarantees at the population level}
\label{subsec:pop_level}
The main error metric in this sub-section is the mean-squared-error (MSE) of the estimates $\widehat \theta_{i, j}$:
\begin{align}
\label{eq:mse}
    \mathrm{MSE} := \frac{1}{mn} \sum_{i \in [n], j \in [m]} \left( \widehat \theta_{i, j} - f(u_i, v_j) \right)^2. 
\end{align}
We first discuss in \Cref{subsec:mcar} the behavior of the MSE of the two-sided nearest neighbor method under the simpler setting where $A_{i,j} \stackrel{iid}{\sim} \mbox{Ber}(p)$ for some $0 < p \leq 1$. This is the MCAR assumption mentioned in \Cref{asump_missingness}. Thereafter in \Cref{subsec:mnar} we discuss the performance of the two-sided nearest neighbor algorithm under a more general setting where we show that the algorithm stays minimax optimal even if there is arbitrary/deterministic missingness in some entries of the matrix. 
\subsubsection{Missing completely at random (MCAR)}
\label{subsec:mcar}
Our first result provides a guarantee when the missingness is independent of the underlying means, a setting referred to as MCAR in the matrix completion and causal inference literature.
\begin{assumption}[MCAR missingness]
\label{asump_missingness}The indicators  $A_{i, j}$ are drawn \iid $\trm{Ber}(p)$, and independently of the latent factors and the noise. 
\end{assumption}
As highlighted in prior works, MCAR assumption, while rare in practice, provides an initial understanding of the algorithm's effectiveness as a function of the amount of missingness (captured by a single parameter $p$ in MCAR) and the factors and noise distributions.
We are now ready to state our guarantee. 
\begin{theorem}
\label{thm:main_result}    
Under \cref{asump_row_col,asump_low_rank,asump_bounded_noise,asump_missingness} and for any fixed $\delta \in (0, 1)$,  the MSE of TS-NN($\mbi{\eta}$) satisfies the following bound conditional on $\mathcal{U, V}$,
\begin{align}
    \mrm{MSE} &\leq 
    c_{0,\delta}
    \bigg(\eta_{row}^2 +\eta_{col}^2 + \frac{c}{p\sqrt{m}} + \frac{c}{p\sqrt{n}} \\
    & + \frac{c_{1,\delta} \sigma^2 L^{(d_1+d_2)/\lambda}}{pmn\parenth{\eta^2_{row}-\frac{c}{\sqrt{m}}}^{\frac{d_1}{2\lambda}}\parenth{\eta^2_{col}-\frac{c}{\sqrt{n}}}^{\frac{d_2}{2\lambda}}} \bigg),
\end{align}
 with probability at least $1 - \delta$, where $c_{0,\delta} = \frac{c(1 + \delta/7)}{(1 - \delta/7)^2}$ and $c_{1,\delta} = \frac{c\log(\frac{14}{\delta})}{(1\!-\!\delta/7)}$.
\end{theorem}

\Cref{thm:main_result} provides an explicit upper bound (with a high probability) on the mean-squared error of the two-sided nearest neighbor algorithm. The proof of \Cref{thm:main_result} has been discussed in \Cref{sec: appendix A}. In order to get superior MSE decay rates, we optimize the above upper bound with respect to $\mbi{\eta}$. An immediate consequence of this exercise is the following corollary. 
\begin{corollary}
\label{cor:main_result}
Under \cref{asump_row_col,asump_low_rank,asump_bounded_noise,asump_missingness}  for $n = \omega\parenth{m^{\frac{d_1}{2\lambda + d_2}}}$ and $n = \mathcal{O}\parenth{m^{\frac{2\lambda + d_1}{d_2}}}$, TS-NN($\mbi{\eta}$) with $\eta_{row}=\eta_{col}=\Theta((mn)^\frac{-\lambda}{2\lambda+d_1+d_2})$ achieves the non-parametric minimax optimal rate, 
    \begin{align}
        \mathrm{MSE} = O\parenth{(mn)^{\frac{-2\lambda}{2\lambda + d_1 + d_2}}}. 
    \end{align}
\end{corollary}

The proof of \Cref{cor:main_result} is provided in \Cref{sec: appendix_B}. When the dimension of the row and the column latent space are equal (i.e.\ $d_1 = d_2 = d$), \Cref{cor:main_result} implies that the two-sided nearest neighbor algorithm with $\eta_{row} = \eta_{col} = \Theta((mn)^{-\frac{\lambda}{2(\lambda + d)}})$ achieves the minimax optimal rate in the moderate regime of $n = \omega(m^{\frac{d}{2\lambda + d}})$ and $n = O(m^{\frac{2\lambda + d}{d}})$. In the Lipschitz case for $\lambda = 1$, $d_1= d_2$ and $m=n$, this regime becomes $n = \omega(m^{\frac{d}{2 + d}})$ and $n = O(m^{\frac{2 + d}{d}})$ and the MSE of two-sided nearest neighbor achieves the rate $O(n^{-\frac{2}{d+1}})$ in this moderate regime. Under the same setting \cite{yu2022nonparametric}'s NN estimator achieves the same rate in an intermediate regime using the additional knowledge of column latent factors. This rate is better than the MSE rate of $O(n^{-\frac{2}{d+2}})$ achieved by the user(row) nearest neighbor method under the same setting.  However under some additional convexity assumptions doubly-robust nearest neighbor (\cite{dwivedi2022doubly}) achieves the MSE rate of $O(n^{-\frac{4}{d+4}})$ in $\lambda = 1$ case which is better than that of the two-sided nearest neighbor .

We note that theoretically, the performance (measured in terms of MSE) of TS-NN algorithm is never worse than the row or column nearest neighbor counterparts. This is because one can choose $\eta_{col}$ small enough so that $\cn_{col}(j) = \{j\}$ for all $j \in [m]$. Then the two-sided nearest neighbor method simplifies to the row nearest neighbor algorithm. Similarly by choosing $\eta_{row}$ small enough TS-NN recovers the column NN algorithm.

\subsubsection{Missing not at random (MNAR)}
\label{subsec:mnar}
The minimax optimality of the two-sided nearest neighbor algorithm holds in much more general missingness patterns than the MCAR setup. To formalize this claim, we introduce our next assumption.
\begin{assumption}[MNAR missingness]
\label{assump:MNAR}
Conditioned on the latent factors $\mc U, \mc V$, the indicators  $A_{i, j}$ are drawn from $\mbox{Ber}(p_{i,j})$ independent of each other and independent of all other randomness.
\end{assumption}
Note that such an assumption, used in prior works like \cite{agarwal2021causal}, allows the missingness to depend on unobserved latent factors, and thus falls under the category of missing not at random.

Next we require a sufficient condition on the number of neighbors. Recall the definitions of $\cn_{row}(i), \cn_{col}(j), \deno$ discussed in Steps $1, 2$ of the TS-NN($\mbi{\eta}$) algorithm in \Cref{sec:algorithm}. 
\begin{assumption}[Minimum number of nearest neighbors]
\label{assump:min_row_col}
There exists a function $g:(0, 1] \to \real_+$ such that for any $\delta \in [0, 1)$ the event $E_{\delta}$ defined as
\begin{align}
E_{\delta} = \bigcap_{(i,j) \in [n] \times [m]} \left\{ \frac{|\deno|}{ |\cn_{row}(i)||\cn_{col}(j)|} \geq g(\delta) \right\},
\end{align}
satisfies $\mathbb{P}(E_{\delta}| \mathcal{U}, \mathcal{V}) \geq 1 - \delta$.
\end{assumption}
\Cref{assump:min_row_col} essentially guarantees the presence of a certain minimum number of nearest neighbors, which in turn aids in carrying out valid statistical inference. 

The following theorem discusses the performance of the algorithm in the general setup. 
\begin{theorem}
    \label{thm:general result}
Under \cref{asump_row_col,asump_low_rank,asump_bounded_noise,assump:MNAR,assump:min_row_col} for any fixed $\delta \in (0,1)$, the MSE of TS-NN($\mbi{\eta}$) the following bound conditional on $\mathcal{U, V}$,
\begin{align}
    \mrm{MSE} &\leq 
    c_{0,\delta}'
    \bigg(\eta_{row}^2 +\eta_{col}^2 + \frac{c}{\sqrt{\Bar{\mbi{p}}_{i,i'}m}} + \frac{c}{\sqrt{\Bar{\mbi{p}}_{j,j'}n}} \\
    & + \frac{c_{1,\delta} \sigma^2 L^{(d_1+d_2)/\lambda}}{mn\parenth{\eta^2_{row}-\frac{c}{\sqrt{m}}}^{\frac{d_1}{2\lambda}}\parenth{\eta^2_{col}-\frac{c}{\sqrt{n}}}^{\frac{d_2}{2\lambda}}} \bigg).
\end{align}
with probability at least $1 - \delta$, where $ c_{0,\delta}' = \frac{c}{g(\delta)(1 - \delta/7)}$ and $ c_{1, \delta} = \frac{c\log(\frac{14}{\delta})}{(1\!-\!\delta/7)}$.
\end{theorem}
The proof of \Cref{thm:general result} has been discussed in \Cref{sec: appendix_C}. To our knowledge, no theoretical analysis exists for any method when both the row and column latent factors are unknown, the latent function is non-linear, and $(\lambda, L)$ \Holder-continuous with $\lambda <1$ (i.e., the function is not Lipschitz), and the missingness is not at random. To this end, \Cref{thm:general result} is a first result of its kind. 

It is crucial to note that \Cref{thm:general result} holds conditional on both the set of row latent factors $\mathcal{U}$ and the set of column latent factors $\mathcal{V}$. The major difference in the MSE bound in the MNAR case as compared to that in \Cref{thm:main_result} is the additional factor of $g(\delta)$ in the denominator of $c_{0, \delta}'$. It is easy to check that $g(\delta) = (1 - \delta)p$ in the MCAR set-up. If we apply \Cref{thm:general result} to the MCAR setup we get the terms $c/(p^2\sqrt{m})$ and $c/(p^2\sqrt{n})$ in the upper bound of MSE in contrast to the terms $c/(p\sqrt{m})$ and $c/(p\sqrt{n})$ which appear in the upper bound in \Cref{thm:main_result}. Thus we get a slightly weaker result than \Cref{thm:main_result}. This is essentially the price which we pay for replacing the stringent \Cref{asump_missingness} with the much more relaxed \Cref{assump:MNAR} and \Cref{assump:min_row_col}. 


As before, if we optimize the upper bound in \Cref{thm:general result} with respect to $\mbi{\eta}$ to obtain explicit MSE decay rates. When $g(\delta) \geq c$, we can once again deduce that for $n = \omega(m^{\frac{d_1}{2\lambda + d_2}})$ and $n = O(m^{\frac{2\lambda + d_1}{d_2}})$, TS-NN($\mbi{\eta}$) with $\eta_{row}=\eta_{col}=\Theta((mn)^\frac{-\lambda}{2\lambda+d_1+d_2})$ achieves the minimax rate $\mathrm{MSE} = O\parenth{(mn)^{\frac{-2\lambda}{2\lambda + d_1 + d_2}}}$.
\begin{remark}
    Let us now illustrate a few examples where \cref{assump:min_row_col} is satisfied. First note that if $p_{i, j} \geq p > 0$ for all $(i,j) \in [n] \times [m]$. Then using Chernoff bound, we can show that conditioned on the latent factors, $|\deno| \geq (1 - \delta)p |\cn_{row}(i)||\cn_{col}(j)|$ holds for all $(i,j) \in [n] \times [m]$ with a high probability. Thus \Cref{assump:min_row_col} is satisfied with $g(\delta) = (1 - \delta)p$ and hence TS-NN($\mbi{\eta})$ achieves minimax optimal rate in an intermediate regime. 
Notably, this setting recovers the guarantee of \Cref{cor:main_result} where $p_{i, j} = p$ as a special case.
\end{remark}

\begin{remark}
\label{rem: MNAR example 2}
More generally, note that
\begin{align}
    |\deno| \geq (1 - \delta) \!\!\!\!\!\!\!\!\!\!\!\!\!\!\!\sum_{i' \in \cn_{row}(i), j' \in \cn_{col}(j)} \!\!\!\!\!\!\!\!\!\!\!\!\!\!\! p_{i',j'}
\end{align}
with high probability (follows from the concentration of a sum of weighted Bernoulli random variables (Lemma 2, \cite{dwivedi2022doubly}). Hence, if
\begin{align}
\!\!\!\!\!\!\!\!\!\!\!\!\!\!\!\sum_{i' \in \cn_{row}(i), j' \in \cn_{col}(j)} \!\!\!\!\!\!\!\!\!\!\!\!\!\!\! p_{i',j'} \geq c |\cn_{row}(i)||\cn_{col}(j)|,
\label{eq:cond_p}
\end{align}
holds with a high probability for some constant $c$, \Cref{assump:min_row_col} shall hold and we can apply \Cref{thm:general result} to guarantee the minimax optimality of the TS-NN($\mbi{\eta}$) algorithm. The condition~\cref{eq:cond_p} is pretty general and can arise in many settings. For instance, suppose there are underlying iid variables $B_{i,j} \sim \mbox{Ber}(1/2)$ for all $(i,j) \in [n] \times [m]$, independent of everything else and we have $A_{i,j} = 0$ if $B_{i,j} = 0$, and $A_{i,j} \sim \mbox{Ber}(p_{i,j})$ if $B_{i,j} = 1$. Suppose $p_{i,j} \geq p > 0$ for all $(i,j) \in [n] \times [m]$ such that $B_{i,j} = 1$. It can be easily checked that with a high probability there exists a subset $S \subset \cn_{row}(i) \times \cn_{col}(j)$ such that $|S| \geq (|\cn_{row}(i)||\cn_{col}(j)|)/2$ and $\sum_{i' \in \cn_{row}(i), j' \in \cn_{col}(j)} p_{i',j'} \geq \sum_{(i',j') \in S} p_{i',j'}$ which is greater than or equal to $ p |S| \geq (p/2)|\cn_{row}(i)||\cn_{col}(j)| $. Thus the required condition is satisfied with $c = (p/2)$. These examples suggest that even with around $50\%$ deterministic missingness ($p_{i,j} = 0$) in the data, we will continue to get optimal results by using TS-NN($\mbi{\eta}$) algorithm. To summarise, as long as conditioned on the latent factors $|\deno| \geq g(\delta)|\cn_{row}(i)||\cn_{col}(j)| $ holds for all $(i,j) \in [n] \times [m]$ for some $g(\delta) \in (0,1]$ with a high probability, the minimax optimality of the two-sided nearest neighbor algorithm can be established.
\end{remark}

\subsection{Non-asymptotic guarantee at the row$\times$column level}
\label{subsec:ind_level}
In this sub-section, we present the non-asymptotic guarantee at the row$\times$column level.   
\begin{theorem}
    \label{thm:point_wise_guarantees}
For each $(i,j) \in [n] \times [m]$, under \cref{asump_row_col,asump_low_rank,asump_bounded_noise,assump:MNAR,assump:min_row_col}, for any fixed $\delta \in (0, 1)$, the TS-NN($\mbi{\eta}$) estimate $\widehat \theta_{i,j}$ satisfies the following with probability at-least $1 - 4\delta$,
\begin{align}
(\widehat \theta_{i,j} - \theta_{i,j})^2 \leq 2 \parenth{\mathbb{B} + \mathbb{V}},
\end{align}
where,
\begin{align}
 \mathbb{B} =&  \frac{c}{c_{d_1, \lambda}} L^{\frac{2d_1}{d_1 + 2 \lambda}} \left(\eta_{col}^2 + c/\sqrt{n} \right)^{\frac{2\lambda}{d_1 + 2 \lambda}} \\
 &+  \frac{c}{c_{d_2, \lambda}} L^{\frac{2d_2}{d_2 + 2 \lambda}} \left(\eta_{row}^2 + c/\sqrt{m}\right)^{\frac{2\lambda}{d_2 + 2 \lambda}}  , \\
 \mathbb{V} =& \frac{c \sigma^2 \log(2/\delta)}{g(\delta) \parenth{1-\delta}^2 nm \parenth{\frac{\eta_{row}^2-\frac{c}{\sqrt{m}}}{L^2}}^{\frac{d_1}{2\lambda}} \parenth{\frac{\eta_{col}^2-\frac{c}{\sqrt{n}}}{L^2}}^{\frac{d_2}{2\lambda}}},
\end{align}
If $c_{\mathcal{P}, u}, c_{\mathcal{P}, v} >0$ denote the lower bounds on the density of row and column latent factor respectively, the constants $c_{d_1, \lambda}$ and $c_{d_2, \lambda}$ are, 
\begin{align}
    c_{d_1, \lambda} = c_{\mathcal{P}, u}^{\frac{2\lambda}{d_1 + 2 \lambda}}, \quad c_{d_2, \lambda} = c_{\mathcal{P}, v}^{\frac{2\lambda}{d_2 + 2 \lambda}}. 
\end{align}
\end{theorem}
The proof of \Cref{thm:point_wise_guarantees} has been discussed in \Cref{sec:pointwise_bounds}. We note that the bounds in both \Cref{thm:general result} and \Cref{thm:point_wise_guarantees} have similar scaling for the variance term, $\mathbb{V}$. The main difference between the population level guarantee in \Cref{thm:general result} and the row$\times$column level guarantee in \Cref{thm:point_wise_guarantees} lies in the bias term $\mathbb{B}$, 
\begin{align}
    \mathbb{B} = \begin{cases}
      & c_{0,\delta}\left( \eta_{col}^2 + \frac{c}{\sqrt{n}}  \right. \\
     & \left. + \eta_{row}^2 + \frac{c}{\sqrt{m}}\right) \quad \mbox{(\Cref{thm:general result})},\\
     & \\
     & \frac{c}{c_{d_1, \lambda}} L^{\frac{2d_1}{d_1 + 2 \lambda}} \left(\eta_{col}^2 + c/\sqrt{n} \right)^{\frac{2\lambda}{d_1 + 2 \lambda}} \\
 &+  \frac{c}{c_{d_2, \lambda}} L^{\frac{2d_2}{d_2 + 2 \lambda}} \left(\eta_{row}^2 + c/\sqrt{m}\right)^{\frac{2\lambda}{d_2 + 2 \lambda}} \quad \mbox{(\Cref{thm:point_wise_guarantees})}. 
    \end{cases}
\end{align}
We have hidden the dependence on the missingness probabilities $p_{i.j}$ in the bias representations above. The bias term in \Cref{thm:point_wise_guarantees} suffers from the curse of dimensionality with respect to $\eta_{row}, \eta_{col}$ (the tuning parameters) and $m,n$ (the size of the matrix). However the bias term in \Cref{thm:general result} does not suffer from this issue. These results agree with our intuition that bias at the entry-wise level is higher than that at the population level. It is interesting to see that the dependence of the bias term at the row$\times$column level on the smoothness of the latent function ($\lambda$), and on the dimensionality of the row and column latent space ($d_1, d_2$) vanishes on averaging the bias across all the entries at the population level.

It is important to note that \Cref{thm:point_wise_guarantees} is the first result in the literature to obtain concentration bounds for the pointwise estimates of TS-NN($\mbi{\eta}$) under a non-linear holder-continuous latent function and under general missingness patterns. \Cref{thm:point_wise_guarantees} generalizes the pointwise-guarantee derived in Theorem-4.1 of \cite{dwivedi2022counterfactual} for row-NN under a non-linear Lipschitz latent factor model. We can optimize the upper bound in \Cref{thm:point_wise_guarantees} with respect to $\mbi{\eta}$ to obtain exact MSE decay rates. When $g(\delta) \geq c > 0$, we can deduce that for $n = \omega(m^{\frac{d_2 + 2\lambda}{4\lambda^2 + 2\lambda d_2 + d_1d_2 }} )$ (this is required to ensure that $\eta_{col}^2 \geq c/\sqrt{n}$) and $n = O(m^{\frac{4\lambda^2 + 2\lambda d_1 + d_1d_2}{d_1 + 2\lambda}})$ (this is required to ensure that $\eta_{row}^2 \geq c/\sqrt{m}$), TS-NN($\mbi{\eta}$) with $\eta_{row} = (mn)^{\frac{-(d_2 + 2\lambda)}{2(2\lambda + d_1 + d_2 + (d_1d_2/\lambda))}}$ and $\eta_{col} = (mn)^{\frac{-(d_1 + 2\lambda)}{2(2\lambda + d_1 + d_2 + (d_1d_2/\lambda))}}$ achieves the rate $(\widehat \theta_{i,j} - \theta_{i,j})^2= O((mn)^{\frac{-2\lambda}{2\lambda + d_1 + d_2 + (d_1d_2/\lambda)}})$ for each $(i,j) \in [n] \times [m]$. The optimal pointwise error rate is understandably slower that the minimax MSE rate obtained in \Cref{subsec:mnar} because of the additional $d_1d_2/\lambda$ term appearing in the optimal pointwise error rate. The term $d_1d_2/\lambda$ governing the gap between the optimal pointwise error rate and the optimal MSE rate suggests that we obtain faster pointwise error rates when we have to search over smaller latent spaces ($d_1, d_2$ small) and when the underlying latent function is more smooth ($\lambda$ is high). 

\subsection{Asymptotic guarantee at the row$\times$column level}
In this sub-section we establish asymptotic normality guarantee for the pointwise estimates $\widehat \theta_{i,j}$. The asymptotic normality result enables us to obtain asymptotically valid confidence intervals for $\widehat \theta_{i,j}$ ($(i,j) \in [n] \times [m]$). We make an asumption regarding the size of sub-sampled nearest neighbors. 
\begin{assumption}[Upper bound on nearest neighbors]
\label{assump:subsample}
A sub-sampling process is done so that the number of nearest neighbours ($|\deno|$) is bounded above by some sequence $\{T_{n,m}\}$ for all $(i,j) \in [n] \times [m]$ which satisfy,
\begin{align}
   T_{n,m}\left\{\eta_{row}^{\frac{4\lambda}{d_2 + 2 \lambda}} + \eta_{col}^{\frac{4\lambda}{d_1 + 2 \lambda}} \right\} = o_P(1) ,
\end{align}
as $m, n \rightarrow \infty$ where the tuning parameters $\mbi{\eta}$ satisfy $\eta_{row}^2 = \Omega(1/\sqrt{m})$ and $\eta_{col}^2 = \Omega (1/\sqrt{n})$. 
\end{assumption}
\Cref{assump:subsample} caps the number of nearest neighbors $\deno$ of the $(i,j)$-th entry at $T_{n, m}$. This is done to ensure that $\deno$ does not grow in an un-restricted manner when $m, n \rightarrow \infty$. This capping ensures that the bias term on being scaled with $\deno$ decays to $0$. We now state our asymptotic guarantee. 
\begin{theorem}
    \label{thm: TSNN CLT}
For each $(i,j) \in [n] \times [m]$, under \cref{asump_row_col,asump_low_rank,asump_bounded_noise,assump:MNAR,assump:min_row_col,assump:subsample}, the TS-NN($\mbi{\eta}$) estimate $\widehat \theta_{i,j}$ satisfies the following distributional convergence,
\begin{align}
    \sqrt{\abss{\deno}}  \parenth{\what{\theta}_{i, j} - \theta_{i, j}} \stackrel{d}{\rightarrow}  \mathcal{N}(0, \sigma^2). 
\end{align}
\end{theorem}
The proof of \Cref{thm: TSNN CLT} has been provided in \Cref{sec:clt_proof}. If $\widehat \sigma^2$ is a consistent estimate of $\sigma^2$ then we can use \Cref{thm: TSNN CLT} to obtain the following asymptotically valid $(1 - \alpha)$ confidence interval for $\theta_{i,j}$ (for $0 <\alpha <1 $), 
\begin{align}
\label{eq: conf interval}
    \parenth{\widehat \theta_{i,j} - \frac{z_{\alpha/2}\widehat \sigma}{\sqrt{\abss{\deno}}}, \widehat \theta_{i,j} +\frac{z_{\alpha/2}\widehat \sigma}{\sqrt{\abss{\deno}}}},
\end{align}
where $z_{\alpha/2}$ is the $1 - (\alpha/2)$-th quantile of the standard normal distribution. 
\begin{remark} \label{rem: consistency of noise sd} We can construct a consistent estimate of $\sigma^2$ by following a similar idea as illustrated in Appendix E.1 of \cite{dwivedi2022counterfactual}. Let us denote the estimate of $\theta_{i,j}$ obtained by TS-NN($\mbi{\eta}$) as $\widehat \theta_{i,j;\mbi{\eta}}$ . We define the following, 
\begin{align}
   \widehat \sigma^2_{n,m;\mbi{\eta}} &= \frac{\sum_{(i,j) \in [n]\times [m]} (X_{i,j} - \widehat \theta_{i,j;\mbi{\eta}})^2 A_{i,j}}{\sum_{(i,j) \in [n]\times [m]} A_{i,j}}.
\end{align}

In the regime $\sum_{(i,j) \in [n]\times [m]}p_{i,j} = \omega(\log(mn))$, it can be shown that $\widehat \sigma^2_{n,m;\mbi{\eta}}$ is a consistent estimate of $\sigma^2$ under the assumptions in \Cref{thm: TSNN CLT}. 
\end{remark}

\begin{remark}
\Cref{thm: TSNN CLT} provides asymptotic normality for a given choice of tuning parameters $\mbi{\eta}$. Getting an asymptotic normality result for the pointwise estimates for an adaptively chosen tuning parameter $\mbi{\eta}$ from the data remains an interesting future direction.   
\end{remark}

\section{Experiments}
\label{sec:experiments}
In this section, we illustrate the practical usability of TS-NN to complement our theoretical findings with empirical evidence via two vignettes: one with synthetic data and another a case study with the real-life dataset HeartSteps. All of our tests have been run on a MacBook Pro with an M2 chip and 32 GB of RAM.

\begin{figure*}[ht]
    \centering
    \begin{tabular}{cc}
         \quad\qquad \textbf{MCAR} & \quad \qquad \textbf{MNAR} \\
        \includegraphics[trim=0in 0.5in 0in 0in, clip, width=0.45\textwidth]{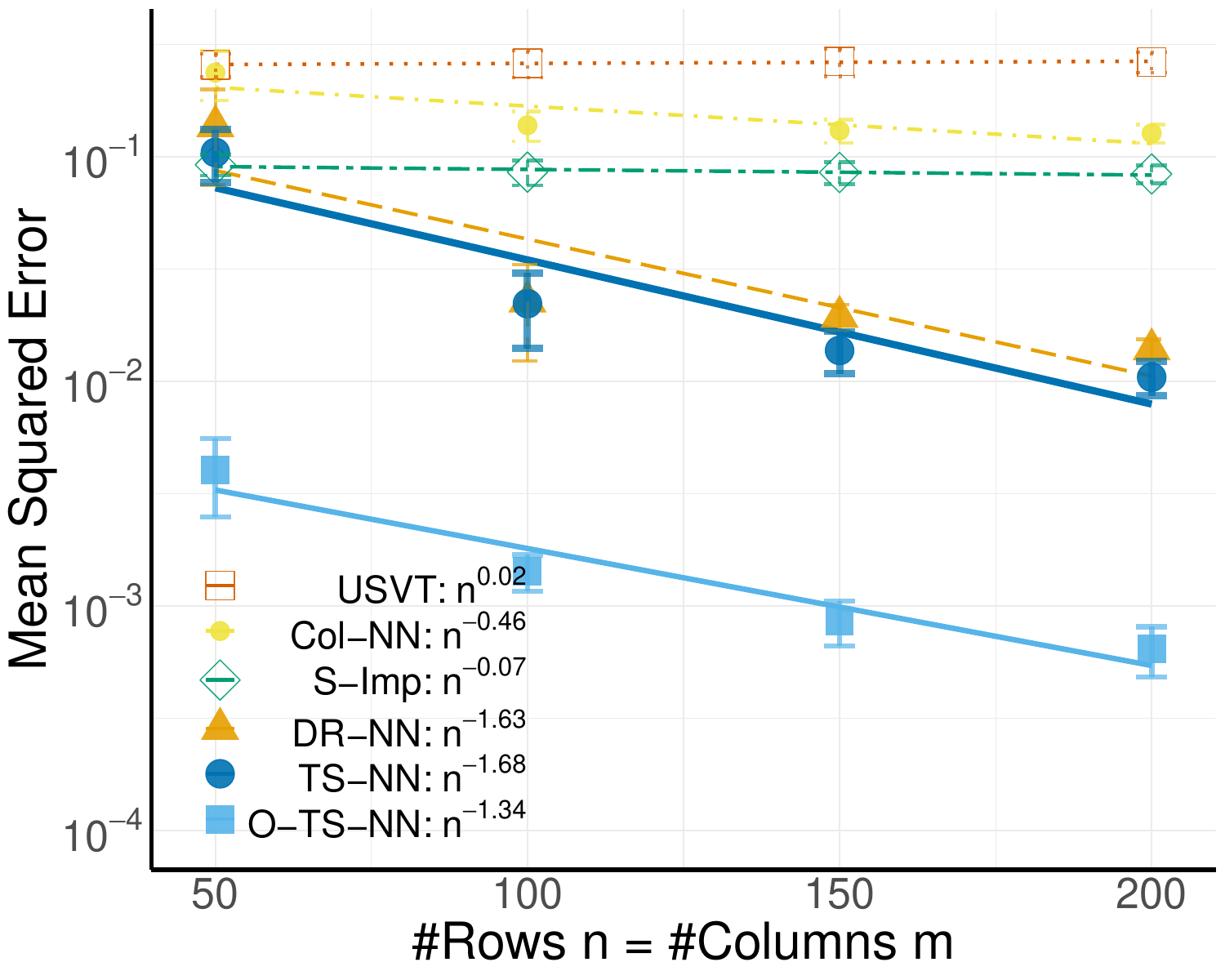} &
        \includegraphics[trim=0in 0.5in 0in 0in, clip, width=0.45\textwidth]{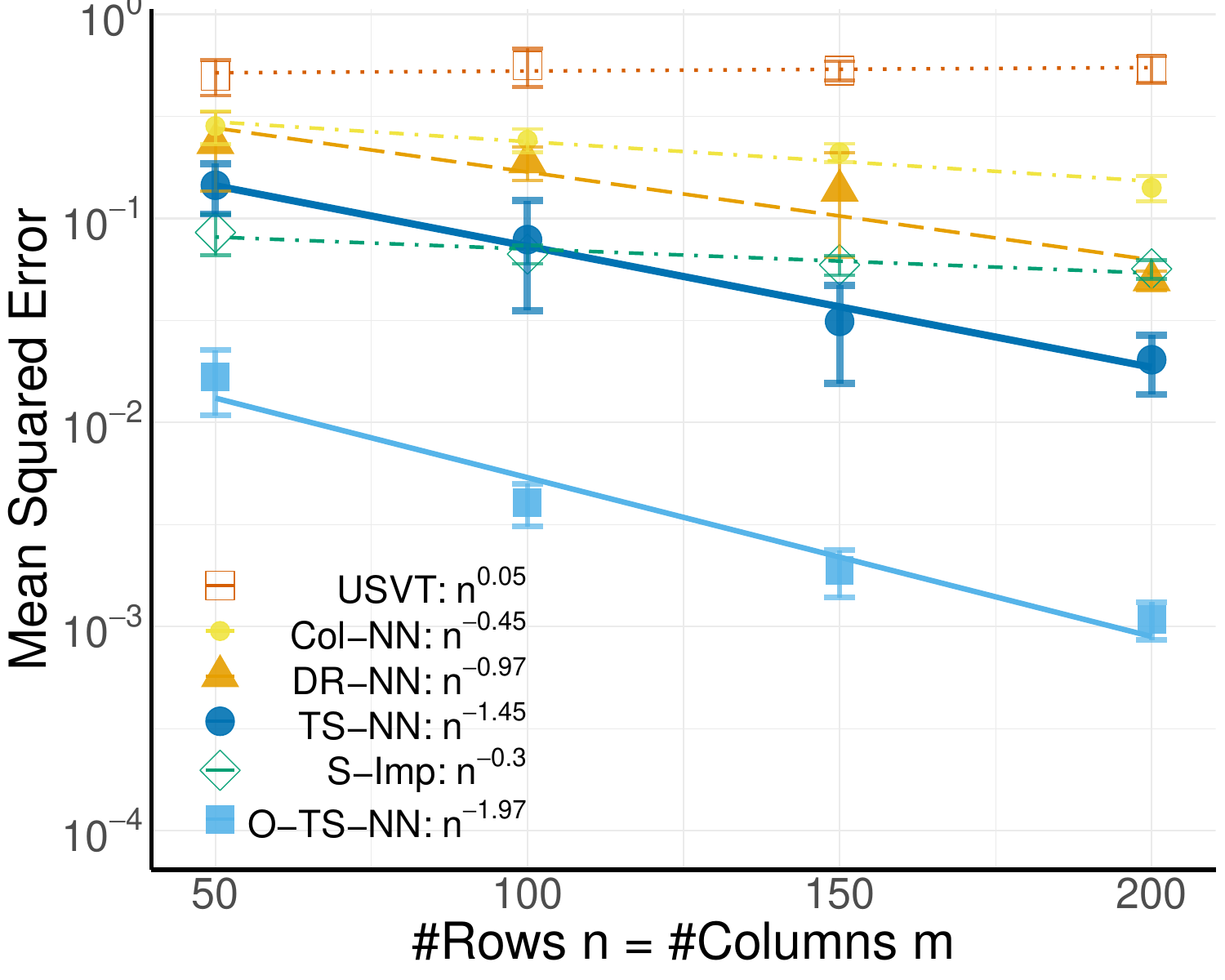}\\[-1mm]
        \ \ \quad \# Rows $n$ &\ \  \quad  \# Rows $n$
    \end{tabular}
    \\[2mm]
    {(a) MSE error rates for various algorithms for $\lambda=0.75$-smooth function and SNR $ \approx 1.41$.}\\
    \rule{\textwidth}{0.5pt}\vspace{1em}
    \begin{tabular}{cc}
         \quad\qquad \textbf{MCAR} & \quad \qquad \textbf{MNAR} \\
        \includegraphics[trim=0in 0.5in 0in 0.5in, clip, width=0.45\textwidth]{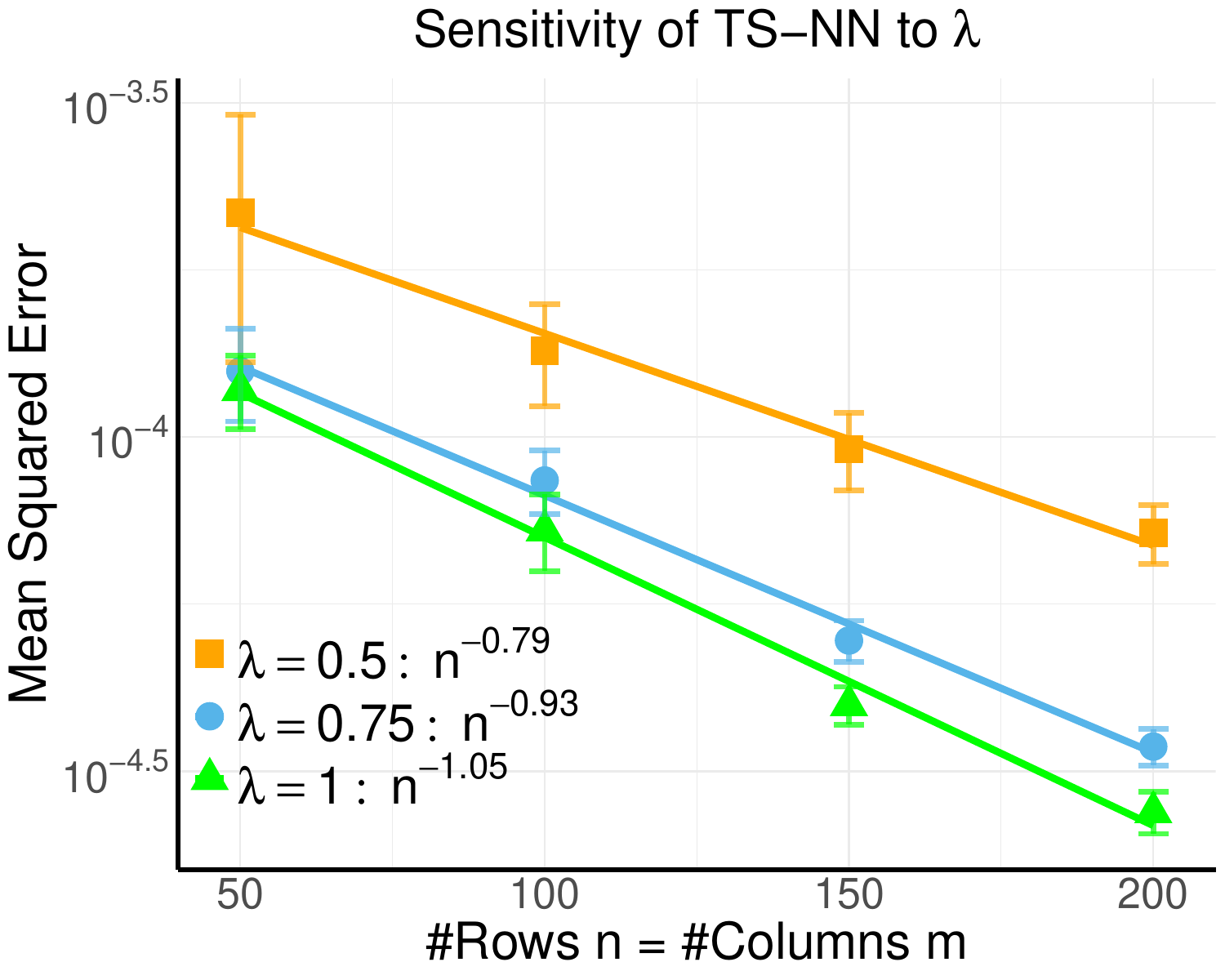} &
        \includegraphics[trim=0in 0.5in 0in 0.5in, clip, width=0.45\textwidth]{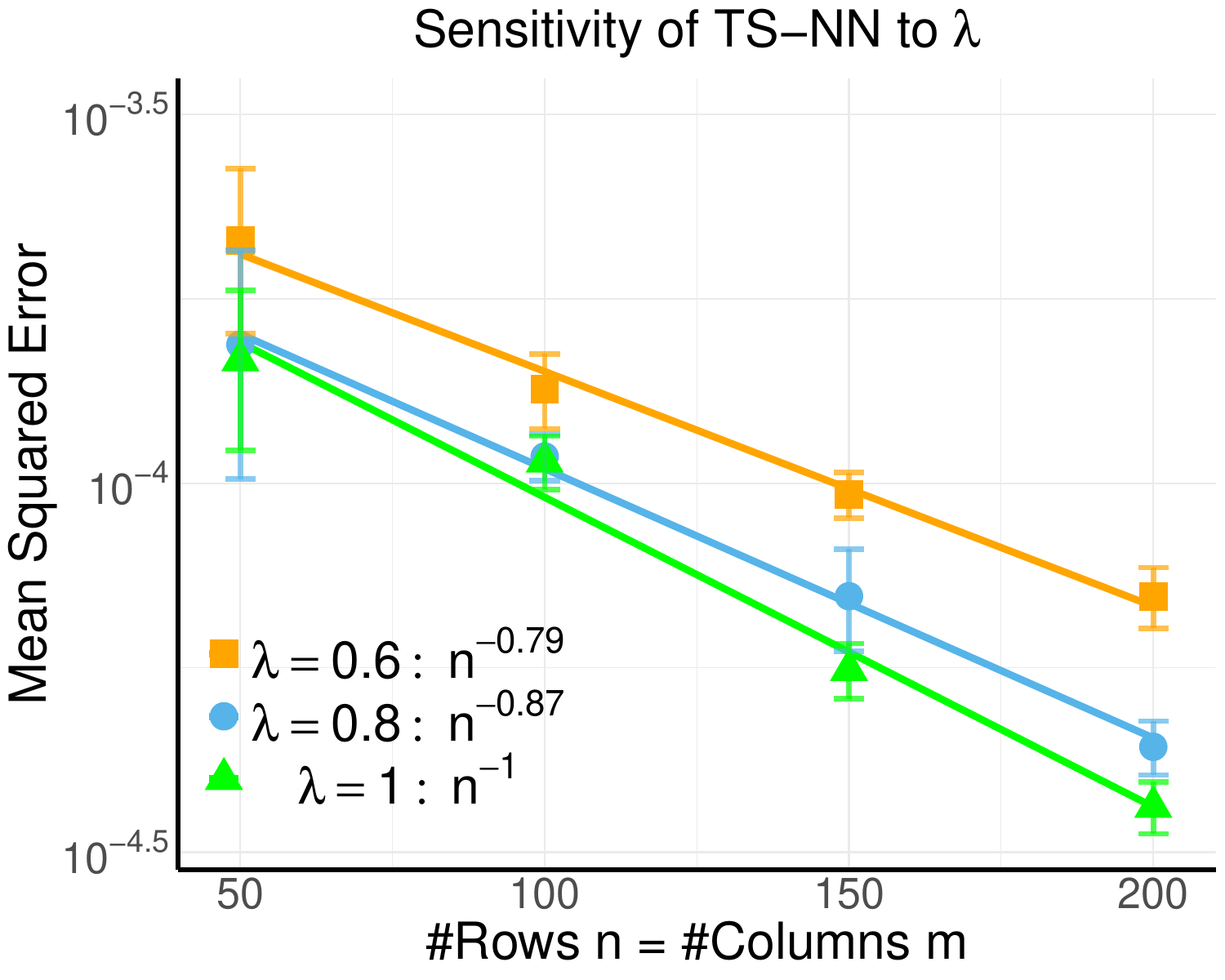} \\[-1mm]
        \ \ \quad \# Rows $n$ &\ \  \quad  \# Rows $n$
    \end{tabular}
    \\[2mm]
    {(b) Variations in TS-NN MSE with smoothness parameter $\lambda$ for SNR $\approx 31$.}
    
    \caption{\tbf{MSE comparison for various benchmarks.} Results are averaged across 10 runs.}
    \label{fig: merged comparison}
\end{figure*}
\subsection{Simulation Study}
\label{subsec: simulation study}
We use the following $(\lambda, 2)$ \Holder smooth $f:\real^2\to\real$:
\begin{align}
    f(u, v) = |u +v|^{\lambda}\mathrm{sgn}(u +v),
\end{align}
and generate \iid latent factors and noise as follows:
\begin{align}
    u_i\sim \Unif[-0.5, 0.5];&\quad v_j\sim \Unif[-0.5, 0.5];\\
   \eps\sim\Gsn(0, \sigma_{\eps}^2);&\quad
    \theta_{i,j} = f(u_i, v_j),
\end{align}
and simulate two distinct missingness mechanisms:
\begin{align}
    &\qtext{MCAR:} A_{i,j} \iidsim \Ber(0.75),\\
    &\qtext{MNAR:}A_{i,j}(u_i, v_j) \sim \Ber(p_{i,j}(u_i, v_j))\\
    &p_{i,j}(u_i, v_j) = \begin{cases}
        0 \qquad &\text{with prob } 0.2\\
        \frac{2}{5} + \frac{1}{5}\mathbb{I}\parenth{u_i + v_j > 0}&\text{with prob } 0.8
    \end{cases}
\end{align}
 The latter formulation, captures an MNAR setting, where data  20\% of entries are deterministically missing and for 80\%  points, larger signals will have a larger probability of being observed. For example, in a movie recommendation system, if a user strongly likes or dislikes a movie, they are more inclined to give a rating than if they have mixed feelings about it.

We present results for a couple of different signal-to-noise ratio in our simulations, which is defined as:
\begin{align}
    \mrm{SNR} = 
    \sqrt{\frac{\sum_{i = 1}^n\sum_{j = 1}^m f^2(u_i, v_j)}{mn\sige^2}}.
\end{align}


\paragraph{Baselines} 
We compare TS-NN with the vanilla NNs, namely row nearest neighbors (Row-NN) and its column counterpart 
 (Col-NN) \cite{li2019nearest,dwivedi2022counterfactual} each of which only use one set of neighbors, and then its doubly robust variant (DR-NN)~\cite{dwivedi2022doubly} which uses both row and column neighbors. We also include the conventional matrix completion methods Universal Singular Value Thresholding (USVT) \cite{Chatterjee15} and SoftImpute \cite{softimpute14}.

We also provide the performance of Oracle TS-NN (O-TS-NN), which has access to both row and column latent factors but doesn't know the $f$. It uses the latent vectors in place of observed rows and columns to compute the inter-row and inter-column distances, eventually yielding the neighborhood set. 


\paragraph{Experiment Setup} We set the number of rows and columns equal, i.e., $m=n$ so that Row and Col-NN perform similarly and we omit one of them. In all NN methods, hyper-parameter tuning is done via cross-validation and SoftImpute is implemented over a log lambda grid and the best MSE was reported (details in \cref{sec: appendix D}). USVT is implemented using the $\texttt{filling}$ R package.
We repeat the experiment 10 times and plot the mean MSE~\cref{eq:mse} along with 1 standard deviation for it (which can be too small to notice), as 
a function of $n$ in \cref{fig: merged comparison} for $\lambda=0.75$ and \cref{fig: lambda = 0.5 benchmark comparison} in Appendix for $\lambda=0.5$. 
We also provide a least squares fit for $\log(\MSE)$ with respect to $\log n$ and report the slope of the regression line in the plot. If the slope of the line is $-0.9$ then, MSE decreases in the order of $n^{-0.9}$. We show it for quantifying the MSE decay of algorithms in simulation studies.

\paragraph{Illustrative Example} We portray the difference between point estimates of TS-NN and SoftImpute(as a standard non-NN matrix completion algorithm) in \cref{fig: scatterplots}. They were compared on a 200$\times 200$ matrix with SNR$^2=2$, MNAR missingness, and $\lambda=0.75$. We plot the true signals $\theta_{i,j}$ on y-axis and their corresponding estimates $\hatij$ on x-axis for each algorithm and look at how they deviate from the $y=x$ line. We highlight the $y=x$ line as a reference because as the value $\abss{\hatij - \signalij}$ gets smaller, the corresponding point comes closer to the $y=x$ line. 

\begin{figure*}[ht]
    \centering
    \begin{tabular}{cc}
         \quad\qquad \textbf{Soft Impute} & \quad \qquad \textbf{TS-NN} \\
        \includegraphics[ width=0.45\textwidth]{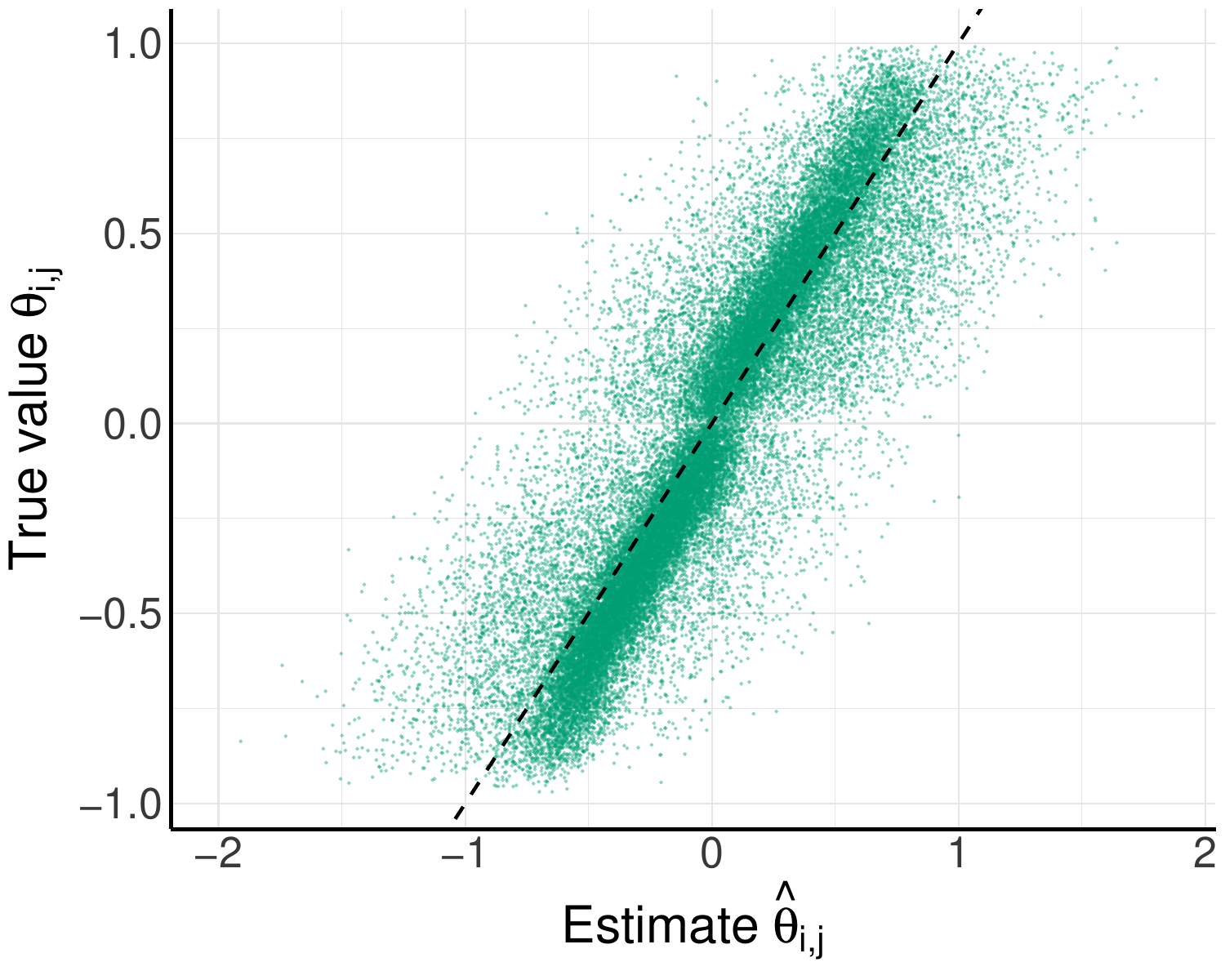} &
        \includegraphics[ width=0.45\textwidth]{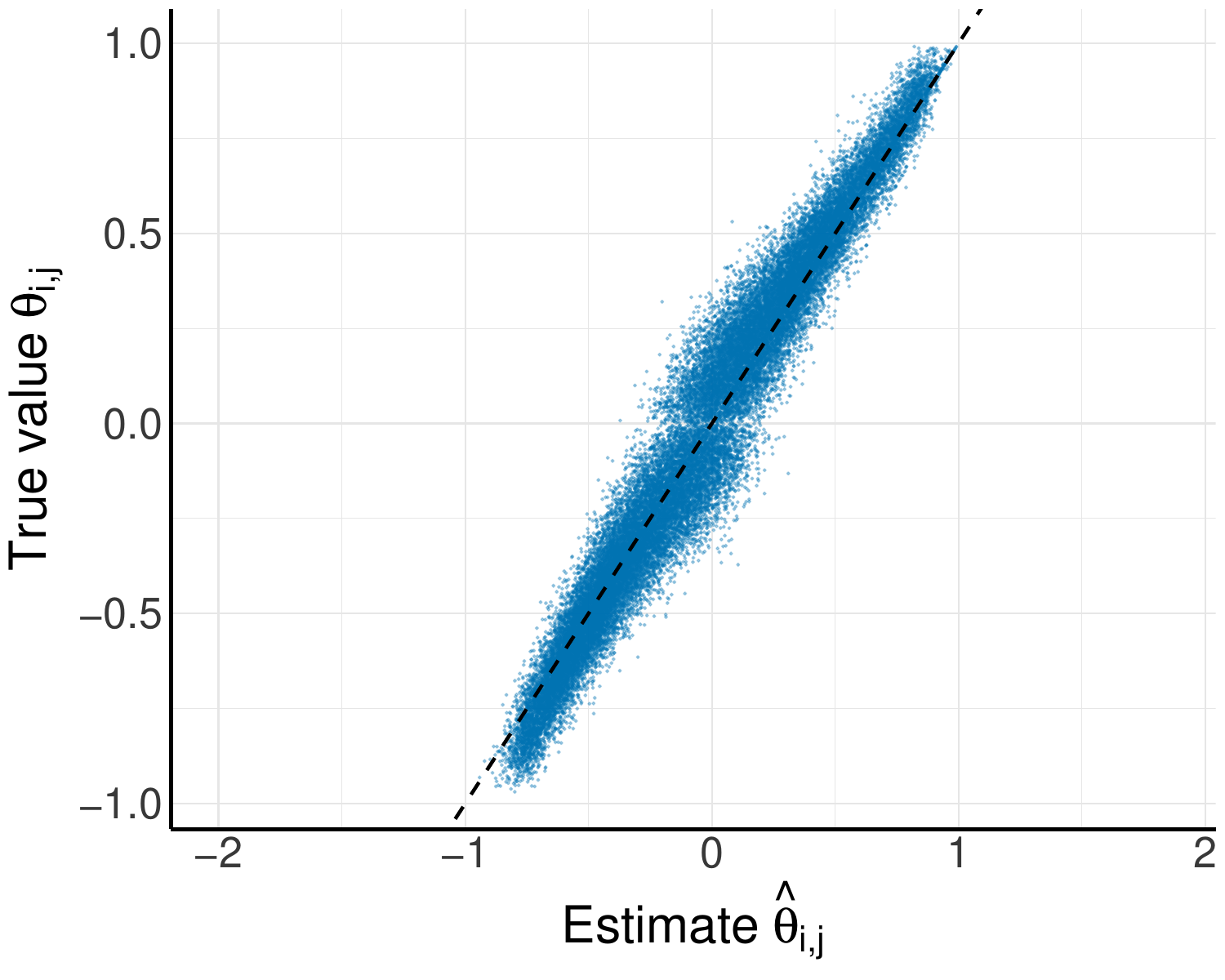}
    \end{tabular}
    \caption{\textbf{Scatterplots of estimated and true signals for TS-NN and SoftImpute under MNAR missingness.} In TS-NN the points are much more concentrated around $y=x$ line, while SoftImpute struggles with controlling the bias. For SoftImpute, in the region $\theta_{i,j}\leq0$, majority of the points lie below the $y=x$ line indicating a pervasive negative bias among it's estimates and vice versa for the region $\theta_{i,j}>0$.} 
    \label{fig: scatterplots}
\end{figure*}

\paragraph{Results} Following our \cref{cor:main_result} and \cref{thm:general result}, the theoretical MSE decay rate of TS-NN becomes $\mathcal{O}\parenth{n^{-\frac{4\lambda}{2\lambda + 2}}}$ in setup where $n= m$ and $d_1 = d_2 = 1$. We note that the empirical MSE decay rates of TS-NN are better than the theoretical ones. TS-NN and its oracle version O-TS-NN show the best MSE decay rates among all the baseline algorithms. Infact in the MCAR setup, TS-NN shows similar MSE decay rate as O-TS-NN for both the lambdas 0.75 and 0.5.

Overall, TS-NN and DR-NN perform better than the one-sided counterparts. On the other hand, USVT and SoftImpute exhibit no-to-weak MSE decay. We observe that only TS-NN maintains its non-trivial error decay while transitioning from MCAR to MNAR setup. While DR-NN is competitive to TS-NN for MCAR, it fairs much worse for MNAR. We highlight that TS-NN exhibits the minimum MSE among all the baselines in both settings for all $n \geq 100$.


\paragraph{Sensitivity to Smoothness of latent function} 
Next, we verify the adaptivity of TS-NN to the smoothness of the underlying signal (quantified by $\lambda$), an implication of \cref{thm:general result} in a high SNR regime to highlight the change in decay rates in low-sample size.  (Note when the SNR is not too high, the effect is not as pronounced, e.g., when we change $\lambda=0.75$ in  \cref{fig: merged comparison}(a) to $\lambda = 0.5$ in \cref{fig: lambda = 0.5 benchmark comparison}.)

Overall, we see an improvement in the estimation accuracy of TS-NN as the smoothness of $f$ increases. Interestingly in MCAR, we obtain empirical MSE decay rates of $n^{-0.79}, n^{-0.93}$ and $n^{-1.05}$ at $\lambda = 0.5, 0.75$ and $1$; the trends is consistent with the theoretical rates of $n^{-0.67}, n^{-0.86}$ and $n^{-1}$ respectively. Theoretical rates are obtained by plugging the $\lambda$ in $\mathcal{O}\parenth{n^{-\frac{4\lambda}{2\lambda + 2}}}$. Even in MNAR, at $\lambda = 0.6, 0.8$ and $1$ we get empirical MSE decay rates of $n^{-0.79}, n^{-0.87}$ and $n^{-1}$ where the theoretical rates are $n^{-0.75}, n^{-0.89}$ and $n^{-1}$ respectively. 


Overall, we observe that TS-NN's performance is adaptive with respect to the smoothness parameter $\lambda$ and the missingness mechanism being MCAR / MNAR, consistent with our theory, namely \cref{cor:main_result} and \cref{thm:general result}. Our simulations provide evidence of the resilience of TS-NN towards arbitrary missingness/intervention patterns while delivering MSE-decay rates similar to that of the oracle-TS-NN. (However, as expected the oracle's performance is strictly better (primarily in terms of constant scaling factors), compared to TS-NN.)

\begin{figure*}[ht]
    \centering
    \begin{tabular}{cc}
         \quad\qquad \textbf{MCAR} & \quad \qquad \textbf{MNAR} \\
        \includegraphics[trim=0in 0.5in 0in 0in, clip, width=0.45\textwidth]{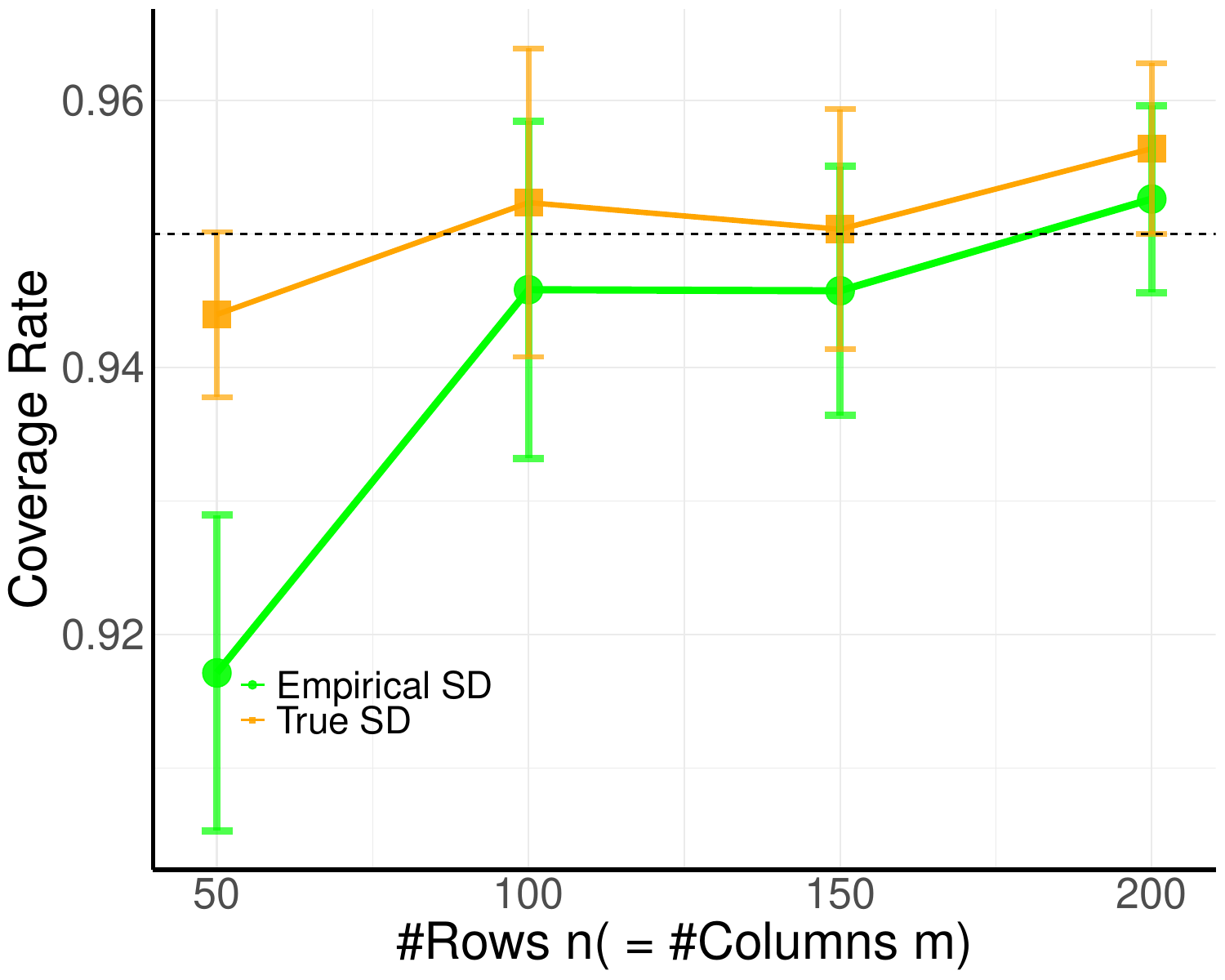} &
        \includegraphics[trim=0in 0.5in 0in 0in, clip, width=0.45\textwidth]{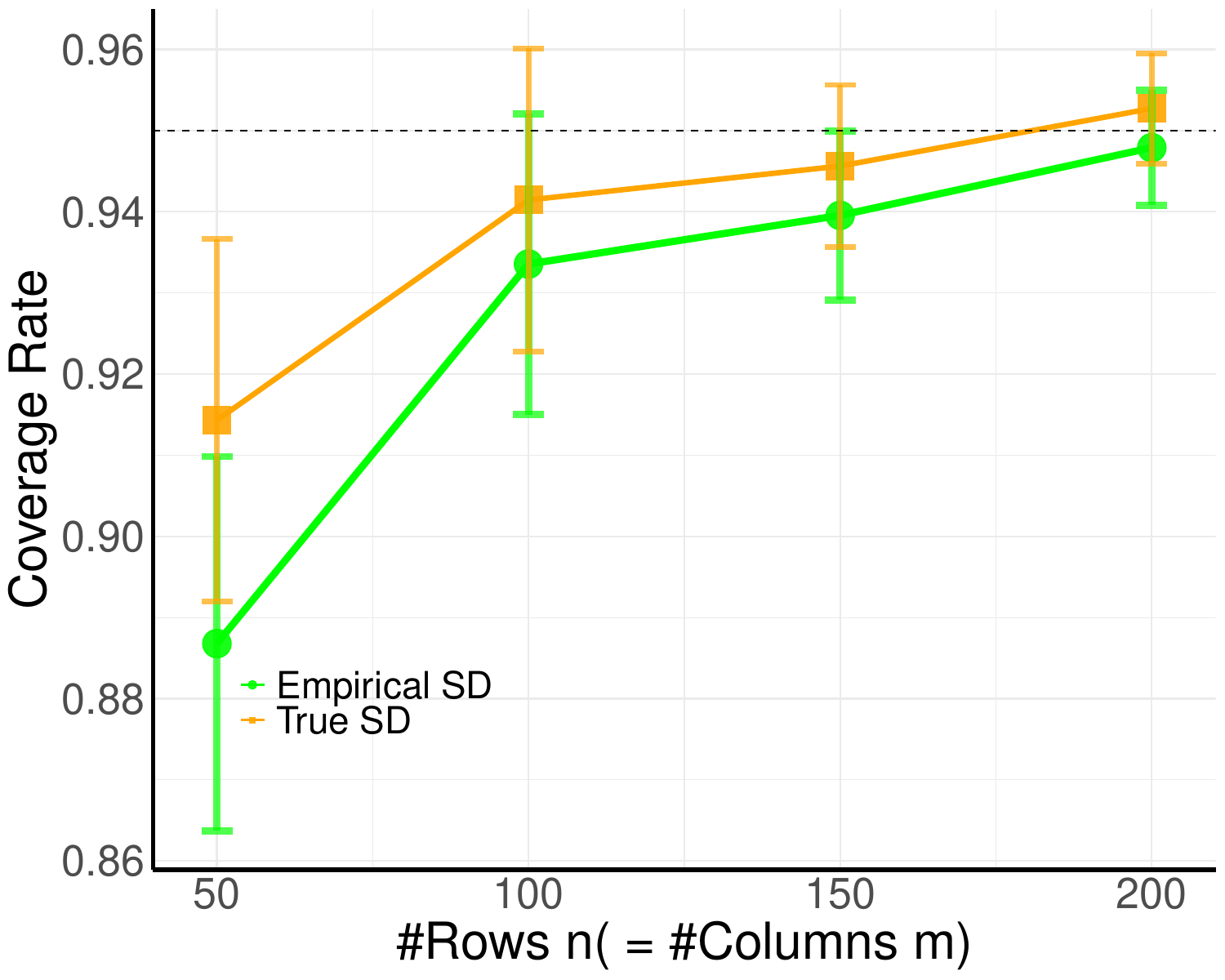}\\[-1mm]
        \ \ \quad \# Rows $n$ &\ \  \quad  \# Rows $n$
    \end{tabular}
    \\[2mm]
    {(a) Empirical study of CLT results for $\lambda=1$-smooth function and SNR $ \approx 1.41$(SNR$^2=2$).}\\
    \rule{\textwidth}{0.5pt}\vspace{1em}
    \begin{tabular}{cc}
         \quad\qquad \textbf{MCAR} & \quad \qquad \textbf{MNAR} \\
        \includegraphics[trim=0in 0.5in 0in 0in, clip, width=0.45\textwidth]{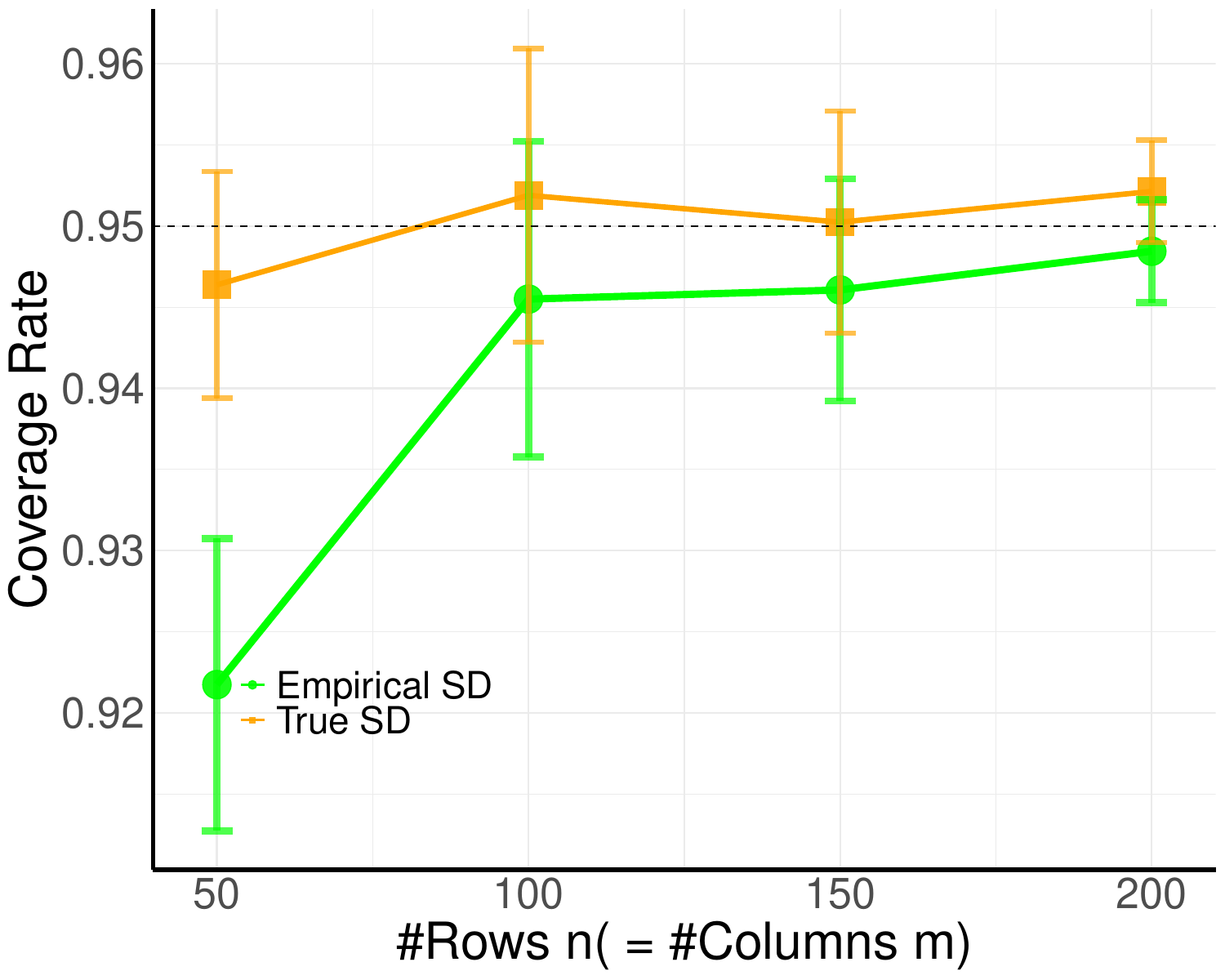} &
        \includegraphics[trim=0in 0.5in 0in 0in, clip, width=0.45\textwidth]{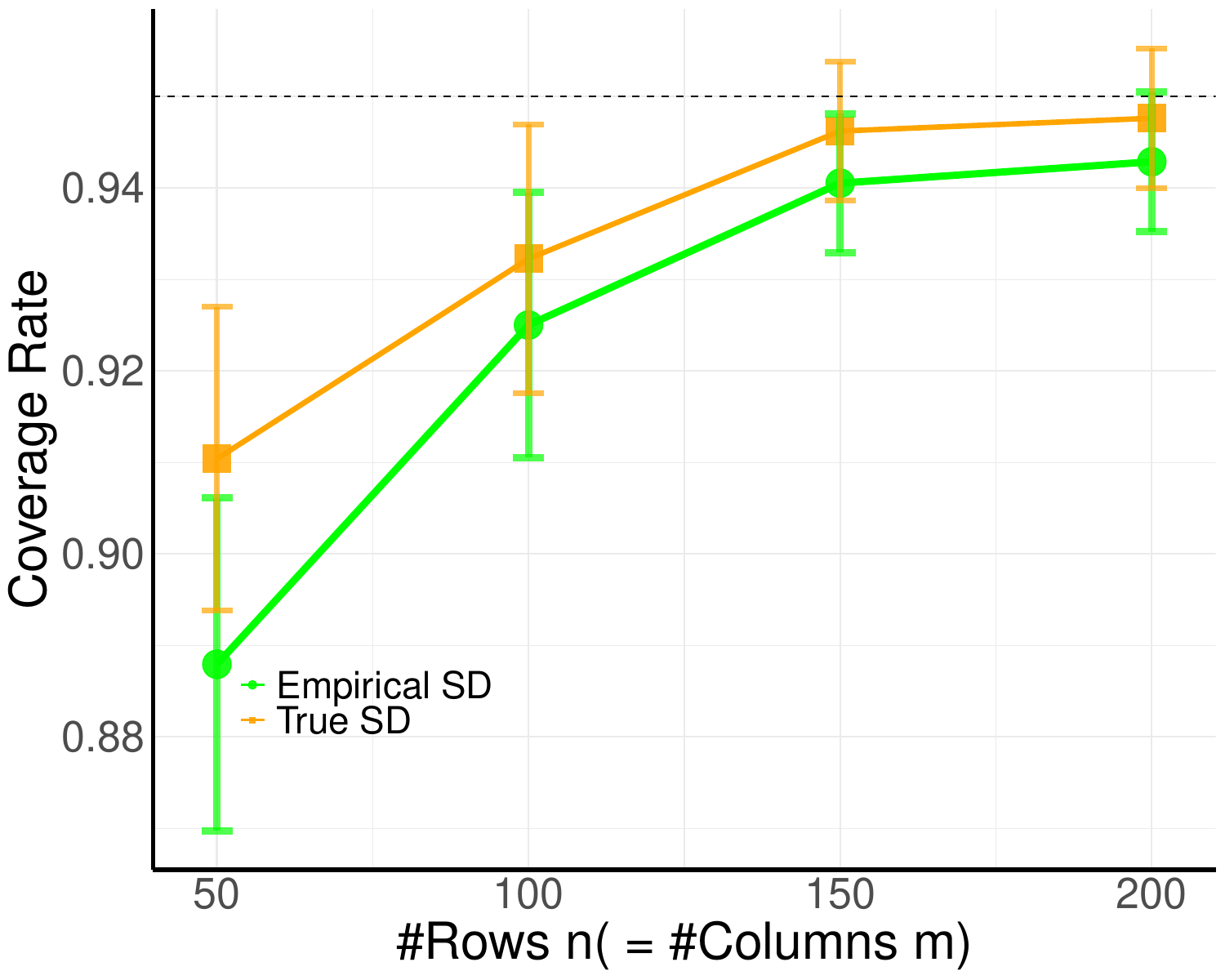} \\[-1mm]
        \ \ \quad \# Rows $n$ &\ \  \quad  \# Rows $n$
    \end{tabular}
    \\[2mm]
    {(b) Empirical study of CLT results for $\lambda=0.75$-smooth function and SNR $ \approx 1.41$(SNR$^2=2$).}
    
    \caption{\tbf{Coverage as a function of $n(=m)$.} Green lines and orange lines indicate coverage given by 95\% CIs obtained using the interval \cref{eq: conf interval} with appropriate  $\what\sig_{\eps}$ and oracle $\sige$ respectively.}
    \label{fig: merged clt comparison}
\end{figure*}

\paragraph{Coverage}

We report the coverage rates of 95\% confidence intervals(CIs) centered at $\what\theta_{i,j}$ as prescribed by \cref{thm: TSNN CLT}. For TS-NN, we split the matrix into 5 folds, and one of them is held out as a test dataset. TS-NN is trained on the remaining 4 folds, and then the TS-NN's coverage rates is evaluated on the test fold. We demonstrate the results for 2 types of CIs, one constructed using oracle $\sige$ and one using estimated $\what{\sig}_{\eps}$, call them $\mbi{CI}_o$ and $\what{\mbi{CI}}$ respectively (\cref{subsubsec: clt appendix}). We refer the reader to \cref{subsec: simulation appendix} for the complete description of data-splitting in TS-NN. We do 5 fold-CV and average the coverage of 95\% $\mbi{CI}_o$ and 95\%  $\what{\mbi{CI}}$ on the test data fold.

We replicate this process 10 times and report the mean and 1 SD bars of empirical coverage of $\mbi{CI}_o$ and $\what{\mbi{CI}}$ as a function of $n(=m)$ in \cref{fig: merged clt comparison}. To boost small sample coverage, we augment our noise SD estimate with the within neighborhood SD. In MCAR setup, we see that 95\% $\what{\mbi{CI}}$ achieves nearly 95\% coverage even for small matrices ($n,m\geq 100$) while 95\% $\mbi{CI}_o$ does the same even for $n,m\geq 50$. In MNAR, both 95\% $\mbi{CI}_o$ and 95\%  $\what{\mbi{CI}}$ need bigger sample sizes ($n,m\geq 150$) for delivering 95\% coverage. The closeness and decreasing gap of coverage for 95\% $\mbi{CI}_o$ and 95\% $\what{\mbi{CI}}$ indicates empirically the consistency of $\what{\sig}_{\eps}$.

\subsection{Case-study with HeartSteps}

\begin{figure*}[ht]
    \centering
    \begin{tabular}{cc}
         \quad\qquad \textbf{Notification Sent} & \quad \qquad \textbf{Notification Not Sent} \\
        \includegraphics[trim=0in 0.5in 0in 0.5in, clip, width=0.45\textwidth]{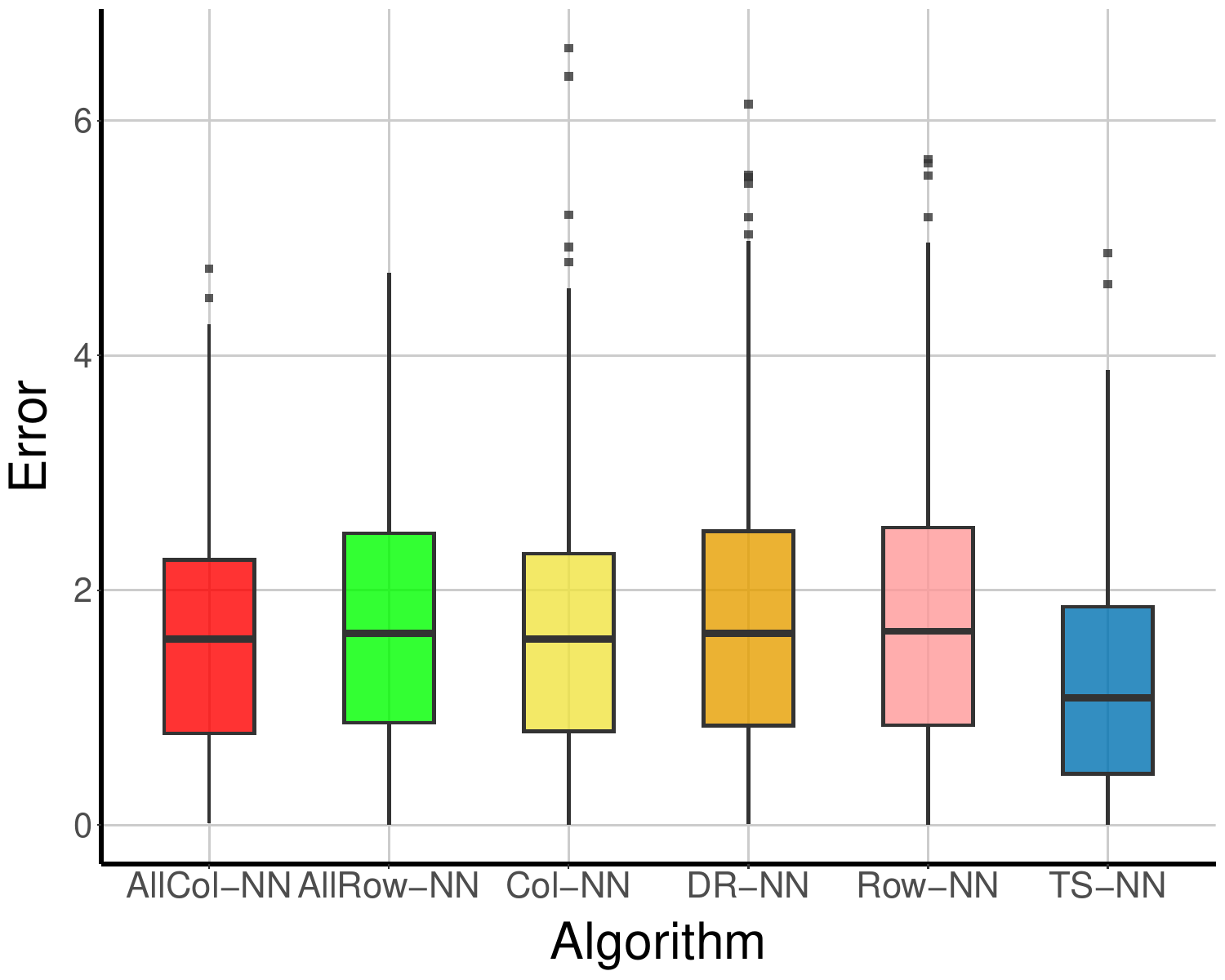} &
        \includegraphics[trim=0in 0.5in 0in 0.5in, clip, width=0.45\textwidth]{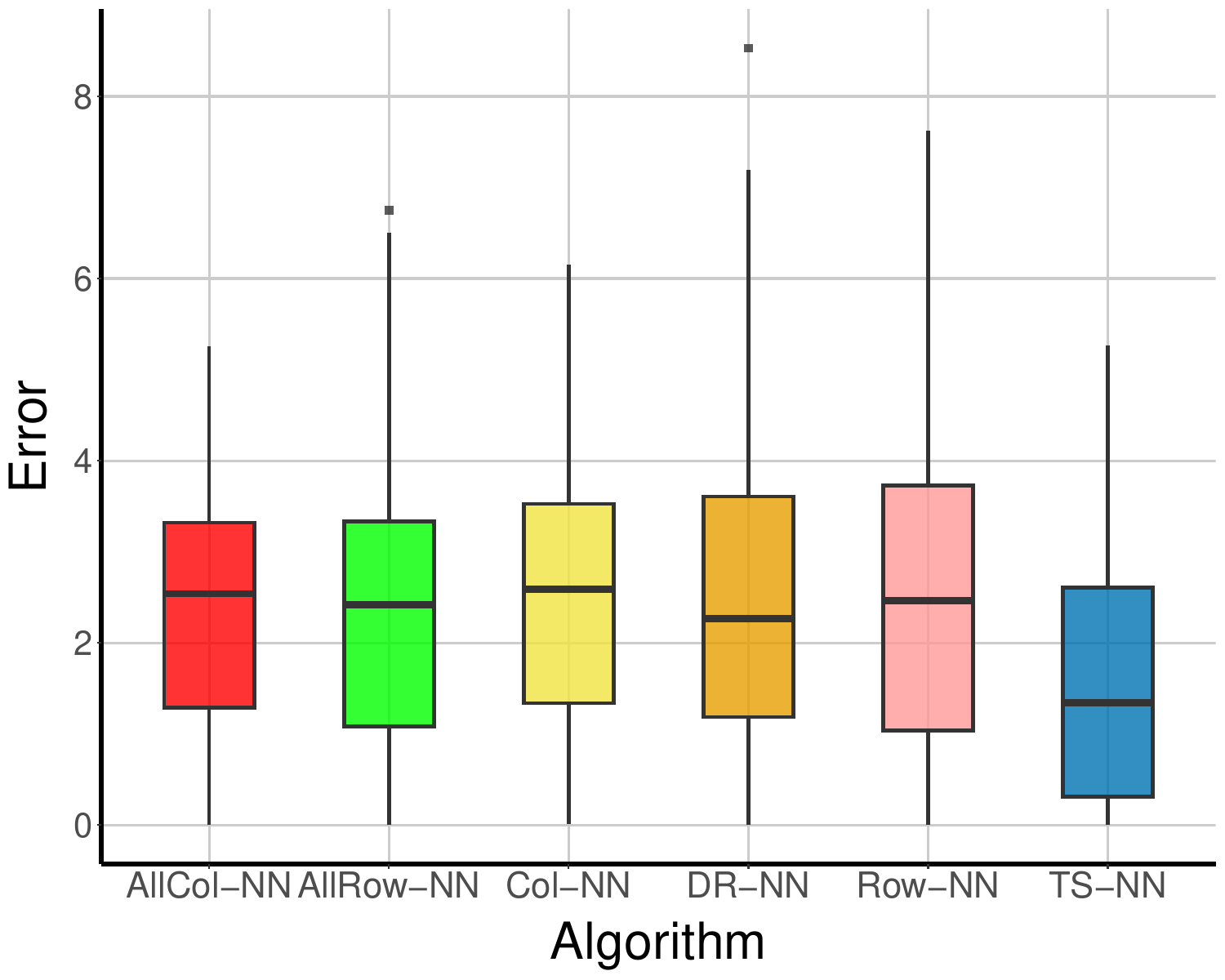}\\[-1mm]
        \ \ \quad Algorithms &\ \  \quad  Algorithms
    \end{tabular}
    \\[2mm]
    \caption{\textbf{Held-out data root MSE (RMSE) for various algorithms  in HeartSteps.} RMSE of USVT and SoftImpute were too high to be competitive and are omitted for plot clarity.}
    \label{fig: HeartStep plots}
\end{figure*}

HeartSteps~\cite{klasnja2019efficacy} is a micro-randomized trial aimed at improving participants' walking activity via mobile notifications, resulting in a health intervention dataset spanning 6 weeks with 37 users. \cite{klasnja2019efficacy} looked for healthy sedentary adults who intended to improve their fitness and walking. They did the recruitment in between August 2015 and January 2016 via fliers and facebook ads. Selected applicants were invited to an interview and were provided with a Jawbone tracker and HeartSteps app installed on their phones for tracking physical activity pre and post-intervention. The underlying recommender system could send notifications (driven by user's context) to user's phone up to five times a day, at user-specific times. For every decision point when the participant is available(refer to \cref{subsec: heartsteps appendix} for details), delivery of the notification was randomized with the following probabilities: 0.4 no notification, 0.3 walking suggestion, and 0.3 anti-sedentary suggestion. This means there is a 60\% chance of sending a notification to a HeartSteps participant at each decision time. The final quantity of interest is the log of the user's step counts within 30 minutes after the decision point.

For us, a matrix of interest will have rows denoting users, columns denoting decision time(for a randomized decision), and entries will be $\log($step counts) after a certain intervention. The intervention is whether a notification (walking or anti-sedentary) was sent or not sent, resulting in 2 types of intervention. 
We will focus on one intervention, say ``notification sent" to understand the generality of missing matrix completion setups that \cref{assump:min_row_col} encapsulates. If a user is available at a decision time point, they receive a notification with a probability of 0.6. Also, a user is available on average around 80\% of the decision times. So, we have around 20\% deterministic missingness in the matrix corresponding to the intervention ``notification sent". This exactly matches the prescription of the example in \cref{rem: MNAR example 2} satisfying \cref{assump:min_row_col} even with deterministic missingness. Following the same chain of arguments, we have condition \cref{eq:cond_p} satisfied over here with $c= 0.6/2 = 0.3$. Therefore, \cref{assump:min_row_col} holds over here, making the HeartSteps a MNAR dataset which can be tackled by TS-NN.

We aim to estimate the counterfactual of both the interventions on each user at every time point, with HeartSteps data from \href{https://github.com/klasnja/HeartStepsV1?tab=readme-ov-file}{this Github repository}. At a macro-level, we want to do matrix completion of the 2 missing matrices arising due to different interventions. We work with $37\times 210$ matrix as majority of the users did not experience more than 210 decision times due to availability issues. 
Also in this real-life case study, we ignore all the context information available in the HeartSteps, which reduces the SNR. We treat this experiment as an empirical demonstration of the usefulness of our methodology for estimation. Since true underlying counterfactuals are unknown to us, we use a 5 fold blocked cross-validation approach to evaluate all the algorithms. We divide the rows/users of the matrix into 5 folds, and in each fold of CV, we hold out the entries in the last 40 decision times of the rows in that particular fold as our test dataset. The remaining entries are used to train our TS-NN and other benchmark algorithms.

For NN-based algorithms, we do not allow an entry to be its self-neighbor (since we are comparing the estimate to the observed entry). After training, the differences between the entries of test dataset and their estimates were recorded and presented as a boxplot in \cref{fig: HeartStep plots}. Both USVT and SoftImpute performed poorly and are dropped from the figure for clarity. For an alternative baseline, we took the simple ``allRow-NN" (and ``allCol-NN") which basically takes all the available rows (and columns) as neighbors. \cref{fig: HeartStep plots} shows how well different NN strategies perform in estimating the held-out test dataset. The small size of the dataset (37 rows and 210 columns) is causing problems for one-sided NNs to show better results than the baseline counterparts. One-sided NNs had a proclivity towards smaller neighborhoods during training which performed poorly during test matrix prediction. Only TS-NN is convincingly beating the baseline algorithms allRow-NN and allCol-NN, with the best median error and least amount of error spread.

\section{Discussion}
\label{sec: Discussion}
We have studied the performance of the two-sided nearest neighbor in the setting where the latent function $f$ is \Holder smooth and the row and column latent factors are unknown. We have seen that it is possible to achieve the optimal minimax non-parametric rate of the oracle algorithm using the two-sided nearest neighbor in certain scalings of rows and columns for a wide range of MCAR and MNAR missingness. Thus the error rate of TS-NN does not suffer from the lack of knowledge of row and column latent factors. The simulations and the real data analysis support the theoretical guarantee derived in this work.

In this work, we analyzed the adaptivity for functions less smooth than Lipschitz. In some settings, functions might have higher-order smoothness (when $f$ belongs in a smooth reproducing kernel Hilbert space). Analyzing whether nearest neighbors adapt to the model smoothness in such settings is an interesting venue for future work.

We assumed independence of row/column latent factors and the exogenous noise, both which are often violated in real life. For example, in movie recommendation systems, a user's perception and rating are both susceptible to their peer group's preference. In other causal panel data settings, the columns of the matrix denote time and the missingness is dependent on the assigned treatments. For settings with such network interference or dependence over time arises due to spillover effects or due to sequentially assigned treatments, designing a correctly adjusted TS-NN is another interesting direction.



\newpage

\appendix





\section{Proof of \Cref{thm:main_result}}
\label{sec: appendix A}
For implementing the algorithm we partition the data-set into two subsets and then use one part for learning the row and column distances, and use the other part for generating final predictions. For improving readability, we perform all the computations on a single data-set (without sample splitting) in the proof. However all the computations will continue to hold even if we partition the data-set. From here-on for notational simplicity, we write $f(i, j)$ for $f(u_i, v_j)$, ${d}^2(i, i' )$ for ${d}_{row}^2(i, i' )$ and ${d}^2(j, j' )$ for ${d}_{col}^2(j, j' )$. In the algorithm, we also perform a sub-sampling procedure after picking the full set of nearest neighbors $\cn_{row}^s(i)$ and $\cn_{col}^s(j)$,
\begin{align}
\label{eq: pre-sub-sampling ngd definition}
        \cn_{row}^s(i) &=\{i' \in [n]: \widehat{d}^2(i, i' )\leq \eta_{row}^2\},\\
        \cn_{col}^s(j) &=\{j' \in [m]: \widehat{d}^2(j, j')\leq \eta_{col}^2\}.
\end{align}

The sub-sampling is done by thresholding $|\cn_{row}^s(i)|$ at $\tau n \eta_{row}^{d_1/\lambda}$ and $|\cn_{col}^s(j)|$ at $\tau m\eta_{col}^{d_2/\lambda}$ where $\tau > 1$. We name the subsampled nearest neighbors from $|\cn_{row}^s(i)|$ and $|\cn_{col}^s(j)|$ as $|\cn_{row}(i)|$ and $|\cn_{col}(j)|$ respectively. Recall the definition of $\deno$ from Step-2 of the TS-NN($\mbi{\eta}$) algorithm in \Cref{sec:algorithm}. For the ease of proof we define the following, 
\begin{align}
\label{eq: denoised estimator}
    \Tilde{\theta}_{i, j}=\frac{\sum_{(i',j')\in \deno}f(u_{i'}, v_{j'})}{|\deno|} = \frac{\sum_{(i',j')\in \deno}f(i', j')}{|\deno|}.  
\end{align}
Note that $ \Tilde{\theta}_{i, j}$ is essentially $\widehat \theta_{i, j}$ where the noisy signals $X_{i',j'}$ appearing in the numerator of $\widehat \theta_{i, j}$ is replaced by the ground-truths $\theta_{i',j'} = f(u_{i'}, v_{j'})$. The MSE can then be decomposed into a bias and a variance term as follows, 
\begin{align}
\label{eq:mse_decomp}
    \mathrm{MSE} &= \frac{1}{mn} \sum_{i \in [n], j \in [m]} \left( \widehat \theta_{i, j} - f(i, j) \right)^2 \\
    &= \frac{1}{mn}\sum_{i \in [n], j \in [m]}  \left(\left(\Tilde{\theta}_{i, j} - f(i, j) \right) + \left(\widehat \theta_{i, j} - \Tilde{\theta}_{i, j} \right)   \right) ^2 \\
    &\leq \frac{2}{mn}\sum_{i \in [n], j \in [m]} \left(\Tilde{\theta}_{i, j} - f(i, j) \right)^2 + \frac{2}{mn} \sum_{i \in [n], j \in [m]}\left(\widehat \theta_{i, j} - \Tilde{\theta}_{i, j} \right)^2 \\
    &= 2  \mathbb{B} + 2 \mathbb{V}. 
\end{align}
We prove the theorem by bounding the bias ($\mathbb{B}$) and the variance ($\mathbb{V}$) term separately. The main tool that we repeatedly use for bounding these terms is the following distance concentration lemma (see proof in \cref{proof_of_lem:nn_conc_lemma}).
\begin{lemma}[Distance concentration]
    \label{lem:nn_conc_lemma}
Under \cref{asump_row_col,asump_low_rank,asump_bounded_noise,assump:MNAR}, for any $\delta \in (0,1]$ we have, 
\begin{align}
\label{eq:nn_conc_prob}
    \mathbb{P}(E_1 \cap E_2 | \mathcal{U}, \mathcal{V}) \geq  1- \delta,
\end{align}
where,
    \begin{align}
    \label{eq:nn_conc_bound}
   E_1 = & \left\{   \sup_{i \neq i'}|\widehat d^2(i,i') - d^2(i,i')| \leq \frac{\Delta_r}{\sqrt{\Bar{\mbi{p}}_{i,i'} m}} \right\}, \\
    \quad E_2 = &\left\{ \sup_{j \neq j'}|\widehat d^2(j,j') - d^2(j,j')| \leq \frac{\Delta_c}{\sqrt{\Bar{\mbi{p}}_{j,j'} n}} \right\}.
    \end{align}
Here the mean of the vectors $\mbi{p}_{i,i'} = [p_{i,j}p_{i',j}]_{j = 1}^m$ and $\mbi{p}_{j,j'} = [p_{i,j}p_{i,j'}]_{i = 1}^n$ are denoted by $\Bar{\mbi{p}}_{i,i'}$ and $\Bar{\mbi{p}}_{j,j'}$ respectively. $\Delta_r$ and $\Delta_c$ are constants free of $m,n$. 
\end{lemma}
Note that for MCAR missingness, the bound \cref{eq:nn_conc_bound} simplifies with $\Bar{\mbi{p}}_{i,i'} = p^2$ and $\Bar{\mbi{p}}_{j,j'} = p^2$. We use the assumption of \Holder-continuity of the latent function $f$ (\cref{asump_low_rank}) to obtain lower bounds on the number of nearest rows $\cn_{row}^s(i)$ and columns $\cn_{col}^s(j)$ under the events $E_1, E_2$. 
\begin{lemma}
\label{lem:bound_nn} 
The full set of row and column nearest neighbors before subsampling ($|\cn_{row}^s(i)|$ and $|\cn_{col}^s(j)|$ respectively) satisfy the following bounds,
\begin{align}
\label{eq:nn_bounds}
&\Prob\parenth{ |\cn_{row}^s(i)| \geq \parenth{1-\delta}n\parenth{\frac{\eta_{row}^2-\frac{\Delta_r}{p\sqrt{m}}}{L^2}}^{\frac{d_1}{2\lambda}} \mbox{ } \mathrm{for} \mbox{ } i \in [n]\Bigg| E_1, E_2, \mathcal{U}, \mathcal{V}} \\
&\geq  1-n\exp\parenth{-\frac{\delta ^2n}{2}\parenth{\frac{\eta_{row}^2-\frac{\Delta_r}{p\sqrt{m}}}{L^2}}^{\frac{d_1}{2\lambda}}}, \\
 & \Prob\parenth{ |\cn_{col}^s(j)| \geq \parenth{1-\delta}m\parenth{\frac{\eta_{col}^2-\frac{\Delta_c}{p\sqrt{n}}}{L^2}}^{\frac{d_2}{2\lambda}} \mbox{ } \mathrm{for} \mbox{ } j \in [m]\Bigg| E_1, E_2, \mathcal{U}, \mathcal{V}}\\
 &\geq  1-m\exp\parenth{-\frac{\delta^2m}{2}\parenth{\frac{\eta_{col}^2-\frac{\Delta_c}{p\sqrt{n}}}{L^2}}^{\frac{d_2}{2\lambda}}}.
\end{align}
\end{lemma}

We consider the event of lower bounding the number of nearest rows and columns in the subsampled neighborhoods $\etarow$ and $\etacol$ as follows
\begin{align}
    A_1(z_1) = \bigcap_i \braces{|\etarow|\geq z_1}, 
   \quad A_2(z_2) = \bigcap_j \braces{|\etacol|\geq z_2}.
\end{align}
where $z_1,z_2$ denote the lower bounds of $|\cn_{row}^s(i)|$ and $|\cn_{col}^s(j)|$ derived in \Cref{lem:bound_nn}. We will now show that  $A_1(z_1) = \bigcap_i \{|\cn_{row}^s(i)|\geq z_1\}$ and $A_2(z_2) = \bigcap_j \{|\cn_{col}^s(j)|\geq z_2\}$ under the events $E_1,E_2$. In the regime $\eta_{row}^2 \geq \Delta_r/\sqrt{m}$ and $\eta_{col}^2 \geq \Delta_c/\sqrt{n}$ we have the equivalence, 
\begin{align}
\begin{cases}
   |\cn_{row}^s(i)| \geq z_1  &\longleftrightarrow \quad |\cn_{row}(i)| \geq z_1 ,\\
    |\cn_{col}^s(j)| \geq z_2  &\longleftrightarrow \quad |\cn_{col}(j)| \geq z_2.  
\end{cases}
\end{align}
The above equivalence holds because the thresholding of the size of $\cn_{row}^s(i)$ is at $\tau n \eta_{row}^{d_1/\lambda}$ which is strictly larger than the lower bound $z_1$. Similarly the thresholding of the size of $\cn_{col}^s(j)$ is at $\tau n \eta_{col}^{d_2/\lambda}$ which is strictly larger than the lower bound $z_2$. This equivalence proves our assertion.

\begin{proof}[Proof of \Cref{lem:bound_nn}]
 We show that under event $E_1$, the first probability statement holds true. The proof of the other bound is similar. We observe that, 
 \begin{align}
  |\cn_{row}^s(i)| & \stackrel{\cref{eq: pre-sub-sampling ngd definition},\cref{eq:nn_conc_bound}}{\geq} \sum_{i'\in [n]}\indic{{{d}^2(i,i') + \Delta_r/(p\sqrt{m})\leq \eta_{row}^2}} \\
  & =  \sum_{i'\in [n]}\indic{{d}^2(i,i') \leq \eta_{row}^2 - (\Delta_r/(p\sqrt{m}))} \\
  & \stackrel{(A\ref{asump_low_rank})}{\geq}  \sum_{i' \in [n]}\indic{L^2\|u_i - u_{i'}\|^{2\lambda} \leq \eta_{row}^2 - (\Delta_r/(p\sqrt{m}))} \\
  & = \sum_{i' \in [n]}\mathbb{I}\left[{\|u_i - u_{i'}\| \leq \left(\frac{\eta_{row}^2 - (\Delta_r/p\sqrt{m})}{L^2} \right)^{1/(2\lambda)}} \right]
\end{align}

Since $u_i \sim \Unif([0,1]^{d_1})$ we get the following by applying Chernoff bound [\cite{hagerup1990guided}] on $\|u_i - u_{i'}\|$,
\begin{align}
\mathbb{P}\left( \|u_i - u_{i'}\| \leq \left(\frac{\eta_{row}^2 - (\Delta_r/(p\sqrt{m}))}{L^2} \right)^{1/(2\lambda)}\right) \geq \left(\frac{\eta_{row}^2 - (\Delta_r/(p\sqrt{m}))}{L^2} \right)^{d_1/(2\lambda)}.
\end{align}
This implies that $\cn_{row}(i)$ stochastically dominates $\mbox{Bin}(n, q)$ distribution where, 
\begin{align}
q = \left(\frac{\eta_{row}^2 - (\Delta_r/(p\sqrt{m}))}{L^2} \right)^{d_1/(2\lambda)}.
\end{align} 
The proof of \eqref{eq:nn_bounds} is completed by using Chernoff bound [\cite{hagerup1990guided}] on this binomial random variable and then using union bound to account for all the rows.
\end{proof}

Let us first analyse the variance part. We consider the event $A_3$ for $0 <\delta <1$ where, 
\begin{align}
    A_3  = \braces{|\deno| \geq (1 - \delta)p|\cn_{row}(i)||\cn_{col}(j)| \mbox{ for all } i,j \in [n] \times [m] }
\end{align}
We apply Chernoff bound [\cite{hagerup1990guided}] on the indicator random variables $A_{i,j}$ and union bound (to account for all $i,j \in [n] \times [m]$) to show that, 
\begin{align}
\label{eq: double ngd lower bound in terms of n_row and n_col}
    \mathbb{P}(A_3) \geq 1 - mn\exp \left(- \frac{\delta^2|\cn_{row}(i)||\cn_{col}(j)|p}{2} \right).
\end{align}
Under the events $A_1(z_1),A_2(z_2),A_3$, applying Hoeffding's inequality [\cite{bentkus2004hoeffding}] on the noise terms $\epsilon_{i,j}$, yield the following for all $i, j \in [n] \times [m]$ ,
\begin{align}
\label{eq:var_proof_hoeffding}
     &\Prob\parenth{|\Tilde{\theta}_{i, j}-\widehat{\theta}_{i, j}|> \zeta\bigg|A_1(z_1),A_2(z_2),A_3}\leq 2e^{-\frac{\zeta^2(1-\delta)pz_1z_2}{2\sigma^2}}\\
\implies& \mathbb{P} \left( |\Tilde{\theta}_{i, j}-\widehat{\theta}_{i, j}|^2> \frac{2\sigma^2 \log(2/\delta)}{(1 - \delta)pz_1z_2}\bigg|A_1(z_1),A_2(z_2),A_3\right) \leq \delta,     
\end{align}
for some $0 < \delta <1$. This implies that conditioned on the events $A_1(z_1),A_2(z_2),A_3$, the following bound holds for the variance term $\mathbb{V}$,
\begin{align}
    &\mathbb{P}\parenth{\mathbb{V} \leq \frac{2\sigma^2 \log(2/\delta)}{(1 - \delta)pz_1z_2} \Bigg | A_1(z_1),A_2(z_2),A_3  } \\
     &\stackrel{\eqref{eq:mse_decomp}}{=} \mathbb{P}\parenth{\frac{1}{mn} \sum_{i \in [n], j \in [m]}\left(\widehat \theta_{i, j} - \Tilde{\theta}_{i, j} \right)^2 \leq \frac{2\sigma^2 \log(2/\delta)}{(1 - \delta)pz_1z_2} \Bigg | A_1(z_1),A_2(z_2),A_3  } \\
    &\stackrel{\eqref{eq:var_proof_hoeffding}}{\geq} 1- \delta.
\end{align}
Now that we have managed to bound the variance term ($\mathbb{V}$), we focus on the bias term ($\mathbb{B}$) in the decomposition of MSE. 


We start by decomposing the bias term into further two parts, 
\begin{align}
\label{eq:bias_decomp}
    \mathbb{B} 
    \stackrel{\cref{eq: denoised estimator},\cref{eq:mse_decomp}}{=}&\frac{1}{nm}\sum_{i \in [n], j \in [m]}\parenth{\parenth{\frac{\sum_{i' \in \cn_{row}(i)}\sum_{j'\in\cn_{col}(t)}f\parenth{i' , j'}A_{i' , j'}}{|\deno|}}-f(i, j)}^2\\
    =&\frac{1}{nm}\sum_{i \in [n], j \in [m]}\parenth{\frac{\sum_{i'\in \cn_{row}(i)}\sum_{j'\in\cn_{col}(j)}\parenth{f\parenth{i' , j'}-f\parenth{i, j}}A_{i', j'}}{|\deno|}}^2\\
    \stackrel{(i)}{\leq} &\frac{1}{nm}\sum_{i \in [n], j \in [m]}\frac{1}{|\deno|}\parenth{\sum_{i'\in \cn_{row}(i)}\sum_{j'\in\cn_{col}(t)}\parenth{f\parenth{i', j'}-f\parenth{i, j}}^2A_{i', j'}}\\
    =&\frac{1}{nm}\sum_{i \in [n], j \in [m]}\frac{1}{|\deno|}\parenth{\sum_{i'\in \cn_{row}(i)}\sum_{j'\in\cn_{col}(j)}\parenth{f\parenth{i' , j'}-f(i' , j)+f(i' , j)-f\parenth{i, j}}^2A_{i', j'}}\\
    \stackrel{(ii)}{\leq} &\frac{2}{nm}\sum_{i \in [n], j \in [m]}\frac{1}{|\deno|}\parenth{\sum_{i'\in \cn_{row}(i)}\sum_{j'\in\cn_{col}(j)}\parenth{f\parenth{i', j'}-f(i' , j)}^2A_{i', j'} +\parenth{f(i' , j)-f\parenth{i, j}}^2A_{i', j'}}\\
     =& \mathbb{B}_1 + \mathbb{B}_2.
\end{align}
Here step-$(i)$ follows from the AM-QM (arithmetic mean - quadratic mean) inequality that $(\sum_{i = 1}^n a_i /n)^2 \leq \sum_{i = 1}^n a_i^2/n$. Step-$(ii)$ follows from the basic inequality $(a + b)^2 \leq 2(a^2 + b^2)$. Thus the bias term $\mathbb{B}$ can be further decomposed into two terms viz $\mathbb{B}_1$ and $\mathbb{B}_2$. The $\mathbb{B}_1$ term arises because of averaging the response across the neighboring rows and the $\mathbb{B}_2$ term arises because of averaging the response across the neighboring columns. We bound the bias term by separately obtaining bounds for $\mathbb{B}_1$ and $\mathbb{B}_2$. We start by obtaining an upper bound (with high probability) on $\mathbb{B}_2$ under the events $E_1, E_2, A_1(z_1), A_2(z_2), A_3$.  
\begin{align}
\label{eq: 2nd term of bias}
  \bias_2  \stackrel{\eqref{eq:bias_decomp}}{=}&\frac{2}{nm}\sum_{i \in [n], j \in [m]}\frac{1}{|\deno|}\parenth{\sum_{i' \in \cn_{row}(i)}\sum_{j' \in\cn_{col}(j)}\parenth{f\parenth{i' , j}-f\parenth{i , j}}^2A_{i' , j'}}\\
    =&\frac{2}{nm} \sum_{i \in [n], j \in [m]}\frac{1}{|\deno|} \parenth{\sum_{i' \in \cn_{row}(i)}\parenth{f\parenth{i' , j}-f\parenth{i , j}}^2\sum_{j'\in\cn_{col}(j)}A_{i' , j'}}\\
    \stackrel{\cref{eq: double ngd lower bound in terms of n_row and n_col}}{\leq} & \frac{2}{nm}\sum_{i \in [n], j \in [m]}\frac{1}{(1-\delta)p|\cn_{row}(i)||\cn_{col}(j)|}\parenth{\sum_{i' \in \cn_{row}(i)}\parenth{f\parenth{i' , j}-f\parenth{i , j}}^2|\cn_{col}(j)|p(1+\delta)}.\label{eq: bias2 1st stop}
\end{align}
The inequality follows because with probability at least $ 1- \delta$ we have $\sum_{j' \in \cn_{col}(j)} A_{i' , j'} \leq | \cn_{col}(j)|p(1 + \delta)$ (by Chernoff bound[\cite{hagerup1990guided}] on $A_{i',j'}$'s) for all $i' \in [n]$. 

\begin{align}
    \cref{eq: bias2 1st stop}= & \frac{2}{nm}\sum_{i \in [n]}\sum_{i' \in \cn_{row}(i)}\frac{1+\delta}{(1-\delta)|\cn_{row}(i)|}\parenth{\sum_{ j \in [m]}\parenth{f\parenth{i' , j}-f\parenth{i , j}}^2}\\
    \leq & \frac{2}{n}\sum_{i \in [n]}\sum_{i' \in \cn_{row}(i)} \frac{1+\delta}{(1-\delta)|\cn_{row}(i)|}d^2(i , i')\\
    \stackrel{\cref{eq:nn_conc_prob}}{\leq} & \frac{2}{n}\sum_{i \in [n]}\sum_{i' \in \cn_{row}(i)}\frac{1+\delta}{(1-\delta)|\cn_{row}(i)|}\parenth{\what{d}^2(i , i')+\frac{\Delta_r}{p\sqrt{m}}}\\
    \stackrel{\cref{eq:defn_nrowcol}}{\leq} & \frac{2}{n}\sum_{i \in [n]}\sum_{i' \in \cn_{row}(i)}\frac{1+\delta}{(1-\delta)|\cn_{row}(i)|}\parenth{\eta_{row}^2+\frac{\Delta_r}{p\sqrt{m}}}\\
    = & 2\parenth{\eta_{row}^2+\frac{\Delta_r}{p\sqrt{m}}}\frac{1+\delta}{1-\delta}.
\end{align}

Similarly we can show that with a high probability the following bound holds for the first term of the bias decomposition $\bias_1$ under the regime $\eta_{row}^2 \geq \Delta_r/\sqrt{m}$ and under the events $E_1, E_2, A_1(z_1), A_2(z_2), A_3$, 
\begin{align}
\label{eq:bias_1}
  \bias_1 \stackrel{\eqref{eq:bias_decomp}}{=}  & \frac{2}{mn} \sum_{i \in [n], j \in [m]} \frac{1}{|\deno|} \sum_{i' \in \cn_{row}(i)} \sum_{j' \in \cn_{col}(j)} (f(i', j') - f(i' , j))^2 A_{i' , j'} \\
    \stackrel{\eqref{eq: double ngd lower bound in terms of n_row and n_col}}{\leq} & \frac{2}{mn} \sum_{i \in [n], j \in [m]} \frac{1}{|\cn_{row}(i)||\cn_{col}(j)|(1 - \delta)p} \sum_{i' \in \cn_{row}(i)} \sum_{j' \in \cn_{col}(j)} (f(i', j') - f(i' , j))^2 A_{i', j'} \\
   \stackrel{\eqref{eq:nn_bounds}}{\leq} &   \frac{2}{mn} \sum_{i \in [n], j \in [m]} \frac{1}{\parenth{1-\delta}n\parenth{(\eta_{row}^2-\frac{\Delta_r}{p\sqrt{m}})/L^2}^{\frac{d_1}{2\lambda}}|\cn_{col}(j)|(1 - \delta)p} \sum_{i' \in \cn_{row}(i)} \sum_{j' \in \cn_{col}(j)} (f(i', j') - f(i' , j))^2 A_{i', j'} \\
    \stackrel{(i)}{\leq} &   \frac{2}{mn} \sum_{i \in [n], j \in [m]} \frac{L^{d_1/\lambda}}{\parenth{1-\delta}n\parenth{\eta_{row}^2-\frac{\Delta_r}{p\sqrt{m}}}^{\frac{d_1}{2\lambda}}|\cn_{col}(j)|(1 - \delta)p}  \sum_{j' \in \cn_{col}(j)} (f(i,j') - f(i, j))^2 A_{i , j'} | \cn_{row}(i)|
\end{align}
Here step-$(i)$ follows by aggregating all the terms which belong to set of neighboring rows of $u_i$. We can further simplify the bound as follows, 
\begin{align}
  \eqref{eq:bias_1}  =&  \frac{2}{mn} \sum_{i \in [n], j \in [m]} \frac{L^{d_1/\lambda}}{\parenth{1-\delta}n\parenth{\eta_{row}^2-\frac{\Delta_r}{p\sqrt{m}}}^{\frac{d_1}{2\lambda}}|\cn_{col}(j)|(1 - \delta)p}  \sum_{j' \in \cn_{col}(j)} (f(i,j') - f(i, j))^2 A_{i , j'} | \cn_{row}(i)| \\
    \stackrel{(i)}{\leq} &  \frac{2}{mn} \sum_{j \in [m]} \frac{\tau n \eta_{row}^{d_1/\lambda}L^{d_1/\lambda}}{\parenth{1-\delta}n\parenth{\eta_{row}^2-\frac{\Delta_r}{p\sqrt{m}}}^{\frac{d_1}{2\lambda}}|\cn_{col}(j)|(1 - \delta)p}  \sum_{j' \in \cn_{col}(j)} \sum_{i \in [n]} (f(i,j') - f(i, j))^2A_{i,j'}  \\
    \stackrel{\eqref{eq:nn_conc_prob}, (ii)}{\leq} &  \frac{2}{mn} \sum_{ j \in [m]} \frac{\tau'}{(1 - \delta)^2|\cn_{col}(j)|p}  \sum_{j' \in \cn_{col}(j)} (\sum_{i \in [n]} A_{i,j'}) \left(d^2(j, j') + \frac{\Delta_c}{p\sqrt{n}} \right) \\
    \label{bias_1_2}
    \stackrel{(iii)}{\leq} & \frac{2}{mn} \sum_{ j \in [m]} \frac{\tau'}{(1 - \delta)^2 |\cn_{col}(j)|p}  \sum_{j'\in \cn_{col}(j)} (1+ \delta) np \left(d^2(j, j') + \frac{\Delta_c}{p\sqrt{n}} \right).
\end{align}
The step-$(i)$ is a consequence of the fact that $|\cn_{row}(i)| \leq \tau n \eta_{row}^{d_1/\lambda}$ because of the subsampling procedure. In the inequality $(ii)$ we used the fact that under the regime $\eta_{row}^2 \geq \Delta/\sqrt{m} $ there exists a constant $\tau' > 0$ such that, 
\begin{align}
\frac{\tau n \eta_{row}^{d_1/\lambda}L^{d_1/\lambda}}{\parenth{1-\delta}n\parenth{\eta_{row}^2-\frac{\Delta_r}{p\sqrt{m}}}^{\frac{d_1}{2\lambda}}} \leq \tau' \quad \mbox{for all} \quad n,
\end{align}
 The inequality $(iii)$ holds because of the fact that with probability at least $1 - \delta$ we have $\sum_{i \in [n]} A_{i,j'} \leq (1 + \delta) np$ (using Chernoff bound [\cite{hagerup1990guided}] on $A_{i,j'}$'s) for all $j' \in [m]$.
 \begin{align}
     \eqref{bias_1_2} \stackrel{\eqref{eq:nn_conc_prob}}{\leq} &\frac{2}{m} \sum_{ j \in [m]} \frac{\tau'(1 + \delta)}{(1 - \delta)^2 |\cn_{col}(j)|}  \sum_{j' \in \cn_{col}(j)}  \left(\widehat{d}^2(j,j') + \frac{2\Delta_c}{p\sqrt{n}} \right) \\
     \stackrel{\eqref{eq:defn_nrowcol}}{\leq} & 2\left(\eta_{col}^2 + \frac{2\Delta_c}{p\sqrt{n}} \right)\frac{\tau'(1 + \delta)}{( 1- \delta)^2}. 
 \end{align}

 If we put together all the bounds that we have shown till now, we get that with probability at least $1 - 7 \delta$,
\begin{align}
    \mbox{MSE} &\leq 2 \mathbb{V} + 2 \bias_1 + 2 \bias_2 \\
    &\leq \frac{4\sigma^2 \log(2/\delta)}{(1 - \delta)pz_1z_2} + 4\parenth{\eta_{row}^2+\frac{\Delta_r}{p\sqrt{m}}}\frac{1+\delta}{1-\delta} + 4\left(\eta_{col}^2 + \frac{2\Delta_c}{p\sqrt{n}} \right)\frac{\tau'(1 + \delta)}{( 1- \delta)^2} . 
\end{align}
It can be easily shown that this is equivalent to the statement made in \Cref{thm:main_result}. This completes of the proof of the theorem. 

\subsection{Proof of \Cref{lem:nn_conc_lemma}}
\label{proof_of_lem:nn_conc_lemma}
We prove the distance concentration lemma for the rows. The result for the columns will follow analogously. Recall the definitions of $\widehat d^2(i,i')$ and $d^2(i,i')$ from \Cref{sec:algorithm}, 
\begin{align}
\label{eq:defns_d_hat_prime}
    \widehat d^2(i,i') &= \frac{\sum_{j \in [m]} (X_{i,j} - X_{i',j})^2A_{i,j}A_{i',j}}{\sum_{j \in [m]} A_{i,j}A_{i',j}} - 2 \sigma^2, \\
    d^2(i,i') &= \frac{1}{m} \sum_{j \in [m]} ( f(u_i, v_j) - f(u_{i'}, v_j))^2. 
\end{align}
We also define the population mean over the column latent factors, 
\begin{align}
\label{eq:rho**}
    \rho^*_{i,i'} = \mathbb{E}_{v}[(f(u_i, v) - f(u_{i'}, v))^2| \mathcal{U}].
\end{align}
To prove \Cref{lem:nn_conc_lemma} we derive concentration bounds of both $\widehat d^2(i,i') - \rho^*_{i,i'}$ and $d^2(i,i') - \rho^*_{i,i'}$. \Cref{lem:nn_conc_lemma} then follows by applying triangle inequality on these two concentration bounds. We start by proving the concentration bounds for $\widehat d^2(i,i') - \rho^*_{i,i'}$. We denote the number of columns corresponding to which entries are observed in both the rows $i,i'$ by $T_{i,i'} = \sum_{j \in [m]} A_{i,j}A_{i',j}$. For the purpose of proof we also define the stopping times $t_{(l)}(i,i')$ for observing an entry in both the rows $i,i'$ for the $l$-th time. To put it rigorously we set $t_{(0)}(i,i') = 0$. For $l \geq 1$ we define iteratively, 
\begin{align}
\label{eq:t_l}
  t_{(l)}(i,i') = \begin{cases}
      \min \{ t: t_{(l - 1)}(i,i') < t \leq m \mbox{ such that } A_{i,j}A_{i',j} = 1\} & \mbox{ if such a $j$ exists}, \\
      m + 1 & \mbox{ otherwise}. 
  \end{cases}  
\end{align}
We observe that $\widehat d^2(i,i') - \rho^*_{i,i'}$ has the following representation.
\begin{align}
\label{eq:dhat_expand}
    \widehat d^2(i,i') - \rho^*_{i,i'} & \stackrel{\eqref{eq:defns_d_hat_prime}}{=} \frac{\sum_{j \in [m]} [(X_{i,j} - X_{i',j})^2 - 2 \sigma^2 - \rho^*_{i,i'}] A_{i,j}A_{i',j}}{\sum_{j \in [m]} A_{i,j}A_{i',j}} \\
    &\stackrel{\eqref{eq:t_l}}{=} \frac{\sum_{l = 1}^{T_{i,i'}} \textbf{1}(t_l(i,i') \leq m) [(X_{i,t_l(i,i')} - X_{i', t_l(i,i')})^2 - 2 \sigma^2 - \rho^*_{i,i'}]}{T_{i,i'}} \\
    &= \frac{\sum_{l = 1}^{T_{i,i'}} W_l}{T_{i,i'}}, 
\end{align}
where $W_l = \textbf{1}(t_l(i,i') \leq m)[(X_{i,t_l(i,i')} - X_{i', t_l(i,i')})^2 - 2 \sigma^2 - \rho^*_{i,i'}]$ for $l = 1, \cdots, T_{i,i'}$. By Hoeffding's inequality [\cite{bentkus2004hoeffding}] on the error terms $\epsilon_{i,j}$ we can show that $\max|\epsilon_{i,j}| \leq \sigma \sqrt{2 \log((2mn)/\delta)} = c_{\epsilon}$. This implies that $|W_l|$ is bounded above by $8 D^2$ where $D = M + c_{\epsilon}$. Let us denote the sigma algebra containing all the information upto time $t$ as $\mathcal{F}_t$ for $ t= 1, \cdots, m$. Let us denote the sigma field generated by the stopping time $t_{l}(i,i')$ by $\mathcal{H}_l$. We observe the following, 
\begin{align}
    \mathbb{E}[W_l | H_l, \mathcal{U}] =& \mathbb{E}[\textbf{1}(t_l(i,i') \leq m)[(X_{i,t_l(i,i')} - X_{i', t_l(i,i')})^2 - 2 \sigma^2 - \rho^*_{i,i'}] | H_l, \mathcal{U}] \\
    =& \textbf{1}(t_l(i,i') \leq m) \mathbb{E}[[(X_{i,t_l(i,i')} - X_{i', t_l(i,i')})^2 - 2 \sigma^2 - \rho^*_{i,i'}] | H_l, \mathcal{U}] \\
    \stackrel{\eqref{eq:rho**}}{=}& \textbf{1}(t_l(i,i') \leq m) ( \rho^*_{i,i'} - \rho^*_{i,i'}) \\
    =& 0. 
\end{align}
Thus we conclude that $\{W_l\}_{l = 0}^{\infty}$ conditioned on the row latent factors $\mathcal{U}$ is a bounded martingale difference w.r.t.\ the sigma algebra $\{ \mathcal{H}_l \}_{l = 0}^{\infty}$. We shall use the Azuma martingale concentration result in this set-up. 
\begin{result}[Azuma martingale concentration]
\label{result_azuma}
Consider a bounded martingale difference sequence $\{ S_n \}_{n=1}^{\infty}$ adapted to the filtration $\{ \mathcal{F}_n \}_{n = 1}^{\infty}$ i.e.\ $\mathbb{E}[S_n | \mathcal{F}_{n - 1}] = 0 $ for all $n \in \mathbb{N}$. Suppose $|S_n| \leq M$ for all $n \in \mathbb{N}$. Then the following event holds with probability at least $ 1- \delta$, 
\begin{align}
    \abss{\sumn S_i} \leq M \sqrt{n \log(2/\delta)}. 
\end{align}
\end{result}
We can make the following computations, 
\begin{align}
\label{eq:d_conc}
    &\mathbb{P}\left( | \widehat d^2(i,i') - \rho^*_{i,i'}| \leq 8 D^2 \sqrt{\frac{\log(2/\delta)}{T_{i,i'}}}, T_{i,i'} > 0 | \mathcal{U} \right) \\
    \stackrel{\eqref{eq:dhat_expand}}{=} & \mathbb{P}\left( \left| \frac{\sum_{l = 1}^{T_{i,i'}} W_l}{T_{i,i'}} \right| \leq 8 D^2 \sqrt{\frac{\log(2/\delta)}{T_{i,i'}}}, T_{i,i'} > 0 | \mathcal{U} \right) \\
    = & \sum_{k = 1}^m \mathbb{P} \left( \left| \frac{\sum_{l = 1}^{T_{i,i'}} W_l}{T_{i,i'}} \right| \leq 8 D^2 \sqrt{\frac{\log(2/\delta)}{k}} , T_{i,i'} = k | \mathcal{U}\right) \\
    \geq & \sum_{k = 1}^m \mathbb{P} \left( \left| \frac{\sum_{l = 1}^{T_{i,i'}} W_l}{T_{i,i'}} \right| \leq 8 D^2 \sqrt{\frac{\log(2/\delta)}{k}}  \mbox{ for all } k \in [m], T_{i,i'} = k | \mathcal{U}\right) \\
    = & \mathbb{P} \left( \left| \frac{\sum_{l = 1}^{T_{i,i'}} W_l}{T_{i,i'}} \right| \leq 8 D^2 \sqrt{\frac{\log(2/\delta)}{k}}  \mbox{ for all } k \in [m], T_{i,i'} > 0  | \mathcal{U}\right) \\
    = & \mathbb{P} \left( \left| \frac{\sum_{l = 1}^{T_{i,i'}} W_l}{T_{i,i'}} \right| \leq 8 D^2 \sqrt{\frac{\log(2/\delta)}{k}}  \mbox{ for all } k \in [m] | \mathcal{U}\right) + \mathbb{P} \left(  T_{i,i'} > 0  | \mathcal{U}\right) - 1 \\
    \stackrel{(i)}{\geq} & \mathbb{P} \left(  T_{i,i'} > 0  | \mathcal{U}\right) - m \delta \\
    \stackrel{(ii)}{\geq} & 1 - (m + 1) \delta. 
\end{align}
The inequality in $(i)$ follows from the application of \Cref{result_azuma} and the fact that for any two events $A, B$ defined in the same probability space we have $\mathbb{P}(A \cap B) \geq \mathbb{P}(A) + \mathbb{P}(B) - 1$. In step $(ii)$ we used that the following probability statement holds for any $\delta > 0$,
\begin{align}
   \mathbb{P}\parenth{ T_{i,i'} - 1^T \mbi{p}_{i,i'} > - \sqrt{2 (1^T \mbi{p}_{i,i'} ) \log(1/\delta) } \bigg| \mathcal{U}, \mathcal{V} } \geq 1- \delta.
\end{align}
This probability statement follows from the concentration of sum of weighted Bernoulli random variables (Lemma-$2$ of \cite{dwivedi2022doubly}). If we use this bound on $T_{i,i'}$ in \eqref{eq:d_conc} we can say with probability at least $ 1 - (m+ 1) \binom{n}{2} \delta$, the following event holds for all rows $i,i'$, 
\begin{align}
    |\widehat d^2 (i,i') - \rho^*_{i,i'} | \leq \frac{8 D^2}{\sqrt{m}} \sqrt{\frac{\log(2/\delta)}{\Bar{\mbi{p}}_{i,i'}\left[1 - \sqrt{\frac{2 \log(1/\delta)}{m \Bar{\mbi{p}}_{i,i'}}} \right]}},
\end{align}
where $\Bar{\mbi{p}}_{i,i'} = (1^T \mbi{p}_{i,i'})/m$. The above equation gives us the concentration of $\widehat d^2 (i,i')$ about $\rho^*_{i,i'}$. To get the concentration of $d^2 (i,i')$ about $\rho^*_{i,i'}$ we use the same argument as above by replacing all the probabilities $p_{i,j}$'s with $1$. In particular, we can show  that with probability at least $ 1 - (m+ 1) \binom{n}{2} \delta$, the following event holds for all rows $i,i'$, 
\begin{align}
    | d^2 (i,i') - \rho^*_{i,i'} | \leq \frac{8 D^2}{\sqrt{m}} \sqrt{\frac{\log(2/\delta)}{1 - \sqrt{\frac{2 \log(1/\delta)}{m }} }}.
\end{align}
Combining both the above inequalities we can say that with probability at least $ 1 - 2(m+ 1) \binom{n}{2} \delta$, the following event holds for all rows $i,i'$, 
\begin{align}
    |\widehat d^2 (i,i') - d^2 (i,i') | &\leq \frac{8 D^2}{\sqrt{m}} \left( \sqrt{\frac{\log(2/\delta)}{\Bar{\mbi{p}}_{i,i'}\left[1 - \sqrt{\frac{2 \log(1/\delta)}{m \Bar{\mbi{p}}_{i,i'}}} \right]}} +  \sqrt{\frac{\log(2/\delta)}{1 - \sqrt{\frac{2 \log(1/\delta)}{m }} }} \right) \\
    & \leq \frac{\Delta_r}{\sqrt{\Bar{\mbi{p}}_{i,i'} m}} \quad \mbox{(for a suitable constant $\Delta_r$)}. 
\end{align}
This completes the proof of the lemma. 

\section{Proof of \Cref{cor:main_result}}
\label{sec: appendix_B}
We start with the upper bound on MSE proved in \Cref{thm:main_result}. Thereafter we substitute the values of $z_1, z_2$ in the result and carefully choose the value of the tuning parameters $\eta_{row}, \eta_{col}$ to get the optimal MSE. 
 \begin{align}
    \mbox{MSE} &\leq \frac{4\sigma^2 \log(2/\delta)}{(1 - \delta)pz_1z_2} + 4\parenth{\eta_{row}^2+\frac{\Delta_r}{p\sqrt{m}}}\frac{1+\delta}{1-\delta} + 4\left(\eta_{col}^2 + \frac{2\Delta_c}{p\sqrt{n}} \right)\frac{\tau'(1 + \delta)}{( 1- \delta)^2} \\
    &\leq C \left[ \frac{1}{n\left(\eta_{row}^2 - \frac{\Delta_r}{\sqrt{m}} \right)^{\frac{d_1}{2\lambda}} m \left(\eta_{col}^2 - \frac{\Delta_c}{\sqrt{n}} \right)^{\frac{d_2}{2\lambda}}}+ \eta_{row}^2 + \frac{\Delta_r}{\sqrt{m}} + \eta_{col}^2 + \frac{\Delta_c}{\sqrt{n}} \right]
\end{align}
Here without loss of generality $C$ is used to denote any arbitrary constant. From the expression of MSE it can be seen that under the regime $n = O(m^{\frac{2\lambda + d_1}{d_2}})$ (this ensures $\eta_{row,opt}^2 \geq \frac{C\Delta_r}{\sqrt{m}}$) and $n = \omega(m^{\frac{d_1}{2\lambda + d_2}})$ (this ensures $\eta_{col,opt}^2 \geq \frac{C\Delta_c}{\sqrt{n}}$), the two-sided NN estimator with $\eta_{row, opt}=\eta_{col, opt}=\Theta\parenth{(mn)^\frac{-\lambda}{2\lambda+d_1+d_2}}$ obtains the rate, 
\begin{align}
    \mathrm{MSE}=O\parenth{(mn)^\frac{-2\lambda}{2\lambda+d_1+d_2}}.
\end{align}   
This completes the proof of the corollary. 

\section{Proof of \Cref{thm:general result}}
\label{sec: appendix_C}
We know from the proof of \Cref{thm:main_result} that with probability at least $1 - \delta$ each of the events $E_1, E_2, A_1(z_1), A_2(z_2)$ hold true. Like in the proof of \Cref{thm:main_result} we start by decomposing the MSE into a bias part and a variance part, 
\begin{align}
\label{eq:mse_decomp_mnar}
        \mbox{MSE} &\leq \frac{2}{mn}\sum_{i \in [n], j \in [m]} \left(\Tilde{\theta}_{i, j} - f(i, j) \right)^2 + \frac{2}{mn} \sum_{i \in [n], j \in [m]}\left(\widehat \theta_{i, j} - \Tilde{\theta}_{i, j} \right)^2 \\
    &= 2 \mathbb{B} + 2 \mathbb{V}.
\end{align}
We bound the MSE by obtaining bounds on the bias and the variance term in the decomposition above.
Let us first look at the variance part. We use $A_4$ to denote the event that $|\deno|\Big|\mathcal{U}, \mathcal{V} \geq g(\delta) |\cn_{row}(i)||\cn_{col}(j)|$. From the assumption made in \Cref{thm:general result} (\Cref{assump:min_row_col}) we know that $\mathbb{P}(A_4) \geq 1- \delta$. From our previous discussions we have the following concentration bound using the Hoeffding's inequality [\cite{bentkus2004hoeffding}] on the noise terms $\epsilon_{i,j}$'s under the events $E_1, E_2, A_1(z_1), A_2(z_2), A_4$, 
\begin{align}
\label{eq:hoeffding_mnar}
    \mathbb{P}\parenth{|\widehat \theta_{i, j} - \Tilde{\theta}_{i, j}| > \zeta \Big| E_1, E_2, A_1(z_1), A_2(z_2), A_4} \leq 2 \exp \left\{- \frac{\zeta^2 g(\delta) z_1 z_2}{2 \sigma^2} \right\}.
\end{align}
This allows us to bound the variance term $\mathbb{V}$,
\begin{align}
 & \mathbb{P}\parenth{\mathbb{V} \leq  \frac{2 \sigma^2 \log(2/ \delta)}{g(\delta) z_1 z_2} \Big|  E_1, E_2, A_1(z_1), A_2(z_2), A_4}  \\
   \stackrel{\eqref{eq:mse_decomp_mnar}}{=} &  \mathbb{P} \parenth{\frac{1}{mn} \sum_{i \in [n], j \in [m]}\left(\widehat \theta_{i, j} - \Tilde{\theta}_{i, j} \right)^2 \leq \frac{2 \sigma^2 \log(2/ \delta)}{g(\delta) z_1 z_2} \Big|  E_1, E_2, A_1(z_1), A_2(z_2), A_4 } \\
   \stackrel{\eqref{eq:hoeffding_mnar}}{\geq}& 1 - mn \delta. 
\end{align}

Now let us come to the bias part. From previous discussions we know that the $\mathbb{B}$ term can be bounded above by,  
\begin{align}
\label{eq:bias_decomp_mnar}
   &\mathbb{B} \\
   \stackrel{\eqref{eq:mse_decomp_mnar}}{\leq} &\frac{2}{nm}\sum_{i \in [n], j \in [m]}\frac{1}{|\deno|}\parenth{\sum_{i'\in \cn_{row}(i)}\sum_{j'\in\cn_{col}(j)}\parenth{f\parenth{i', j'}-f(i' , j)}^2A_{i', j'} +\parenth{f(i' , j)-f\parenth{i, j}}^2A_{i', j'}} \\
    =& \mathbb{B}_1 + \mathbb{B}_2.
\end{align}
Let us first put an upper bound (with high probability) on the second term in the bias decomposition, $\mathbb{B}_2$ under the events $E_1, E_2, A_1(z_1), A_2(z_2), A_4$. We can re-write $\mathbb{B}_2$ as,
\begin{align}
\label{eq: 2nd term of bias_1}
    \mathbb{B}_2 = 
    \frac{2}{nm} \sum_{i \in [n], j \in [m]} \parenth{\sum_{i' \in \cn_{row}(i)}\parenth{f\parenth{i' , j}-f\parenth{i , j}}^2\left(\frac{\sum_{j'\in\cn_{col}(j)}A_{i' , j'}}{\sum_{i' \in \cn_{row}(i), j' \in \cn_{col}(j)} A_{i',j'}}\right)}. 
\end{align}
Let us denote the weights by $ w_{i',j'}^{i,j} = A_{i',j'}/(\sum_{i' \in \cn_{row}(i), j' \in \cn_{col}(j)} A_{i',j'})$. We note that, 
\begin{align}
\label{eq:weight_bound}
 w_{i',j'}^{i,j} = \frac{ A_{i',j'}}{\sum_{i' \in \cn_{row}(i), j' \in \cn_{col}(j)} A_{i',j'}} \leq  \frac{1}{\deno} \stackrel{(A\ref{assump:min_row_col})}{\leq} \frac{1}{g(\delta) \cn_{row}(i)\cn_{col}(j)}.  
\end{align}
This implies that, 
\begin{align}
\label{eq:weight_avg_bound}
 \frac{\sum_{j'\in\cn_{col}(j)}A_{i' , j'}}{\sum_{i' \in \cn_{row}(i), j' \in \cn_{col}(j)} A_{i',j'}} = \sum_{j' \in \cn_{col}(j)}  w_{i',j'}^{i,j} \stackrel{\eqref{eq:weight_bound}}{\leq} \frac{1}{g(\delta) \cn_{row}(i)} . 
\end{align}
We use this bound in the expression of the $\mathbb{B}_2$ to get the following, 
\begin{align}
\label{eq:b2_expr}
\mathbb{B}_2 \stackrel{\eqref{eq: 2nd term of bias_1}}{=}& \frac{2}{nm} \sum_{i \in [n], j \in [m]} \parenth{\sum_{i' \in \cn_{row}(i)}\parenth{f\parenth{i' , j}-f\parenth{i , j}}^2\left(\frac{\sum_{j'\in\cn_{col}(j)}A_{i' , j'}}{\sum_{i' \in \cn_{row}(i), j' \in \cn_{col}(j)} A_{i',j'}}\right)} \\
    \stackrel{\eqref{eq:weight_avg_bound}}{\leq} & \frac{2}{g(\delta)nm}\sum_{i \in [n], j \in [m]}\frac{1}{|\cn_{row}(i)|}\parenth{\sum_{i' \in \cn_{row}(i)}\parenth{f\parenth{i' , j}-f\parenth{i , j}}^2}\\
    =  & \frac{2}{g(\delta)nm}\sum_{i \in [n]}\sum_{i' \in \cn_{row}(i)}\frac{1}{|\cn_{row}(i)|}\parenth{\sum_{ j \in [m]}\parenth{f\parenth{i' , j}-f\parenth{i , j}}^2}\\
    \stackrel{\eqref{eq:defns_d_hat_prime}}{=} & \frac{2}{g(\delta)n}\sum_{i \in [n]}\sum_{i' \in \cn_{row}(i)} \frac{1}{|\cn_{row}(i)|}d^2(i , i').
\end{align}
We can further simplify the bound using \Cref{lem:nn_conc_lemma}, 
\begin{align}
 \eqref{eq:b2_expr} \stackrel{L\ref{lem:nn_conc_lemma}}{\leq} & \frac{2}{g(\delta)n}\sum_{i \in [n]}\sum_{i' \in \cn_{row}(i)}\frac{1}{|\cn_{row}(i)|}\parenth{\hat{d}^2(i , i')+\frac{\Delta_r}{\sqrt{\Bar{\mbi{p}}_{i,i'} m}}}\\
    \stackrel{\eqref{eq:defn_nrowcol}}{\leq} & \frac{2}{g(\delta)n}\sum_{i \in [n]}\sum_{i' \in \cn_{row}(i)}\frac{1}{|\cn_{row}(i)|}\parenth{\eta_{row}^2+\frac{\Delta_r}{\sqrt{\Bar{\mbi{p}}_{i,i'} m}}}\\
    = & \frac{2}{g(\delta)}\parenth{\eta_{row}^2+\frac{\Delta_r}{\sqrt{\Bar{\mbi{p}}_{i,i'} m}}}. 
\end{align}

Similarly under the regime $\eta_{row}^2 \geq \Delta_r/\sqrt{m}$ we can bound $\mathbb{B}_1$ under the events $E_1, E_2, A_1(z_1), A_2(z_2), A_4$, 
\begin{align}
\label{eq:b1_mnar}
   \mathbb{B}_1  \stackrel{\eqref{eq:bias_decomp_mnar}}{=}& \frac{2}{mn} \sum_{i \in [n], j \in [m]} \frac{1}{|\deno|} \sum_{i' \in \cn_{row}(i)} \sum_{j' \in \cn_{col}(j)} (f(i', j') - f(i' , j))^2 A_{i' , j'} \\
    \stackrel{(A\ref{assump:min_row_col})}{\leq} & \frac{2}{g(\delta)mn} \sum_{i \in [n], j \in [m]} \frac{1}{|\cn_{row}(i)||\cn_{col}(j)|} \sum_{i' \in \cn_{row}(i)} \sum_{j' \in \cn_{col}(j)} (f(i', j') - f(i' , j))^2  \\
    \stackrel{(L\ref{lem:bound_nn})}{\leq} &   \frac{2}{g(\delta)mn} \sum_{i \in [n], j \in [m]} \frac{1}{\parenth{1-\delta}n\parenth{\frac{\eta_{row}^2-\frac{\Delta_r}{\sqrt{m}}}{L^2}}^{\frac{d_1}{2\lambda}}|\cn_{col}(j)|} \sum_{i' \in \cn_{row}(i)} \sum_{j' \in \cn_{col}(j)} (f(i', j') - f(i' , j))^2 \\
    \stackrel{(i)}{\leq} &   \frac{2}{g(\delta)mn} \sum_{i \in [n], j \in [m]} \frac{1}{\parenth{1-\delta}n\parenth{\frac{\eta_{row}^2-\frac{\Delta_r}{\sqrt{m}}}{L^2}}^{\frac{d_1}{2\lambda}}|\cn_{col}(j)|}  \sum_{j' \in \cn_{col}(j)} (f(i,j') - f(i, j))^2  | \cn_{row}(i)| .
\end{align}
Here step-$(i)$ follows by aggregating all the terms which belong to the set of neighboring rows of $u_i$. We further simplify the bound as follows, 
\begin{align}
    \eqref{eq:b1_mnar} \stackrel{(ii)}{\leq} &  \frac{2}{g(\delta)mn} \sum_{j \in [m]} \frac{\tau n \eta_{row}^{d_1/\lambda}}{\parenth{1-\delta}n\parenth{\frac{\eta_{row}^2-\frac{\Delta_r}{\sqrt{m}}}{L^2}}^{\frac{d_1}{2\lambda}}|\cn_{col}(j)|}  \sum_{j' \in \cn_{col}(j)} \sum_{i \in [n]} (f(i,j') - f(i, j))^2  \\
    \stackrel{(iii), \eqref{eq:nn_conc_prob}}{\leq} & \frac{2}{g(\delta)m} \sum_{ j \in [m]} \frac{\tau'}{(1 - \delta) |\cn_{col}(j)|}  \sum_{j'\in \cn_{col}(j)}  \left(d^2(j, j') + \frac{\Delta_c}{\sqrt{ \Bar{\mbi{p}}_{j,j'}n}} \right) \\
     \stackrel{\eqref{eq:nn_conc_prob}}{\leq} & \frac{2}{g(\delta)m} \sum_{ j \in [m]} \frac{\tau'}{(1 - \delta) |\cn_{col}(j)|}  \sum_{j' \in \cn_{col}(j)}  \left(\widehat{d}^2(j,j') + \frac{2\Delta_c}{\sqrt{\Bar{\mbi{p}}_{j,j'}n}} \right) \\
     \stackrel{\eqref{eq:defn_nrowcol}}{\leq} & \frac{2}{g(\delta)}\left(\eta_{col}^2 + \frac{2\Delta_c}{\sqrt{\Bar{\mbi{p}}_{j,j'}n}} \right)\frac{\tau'}{ (1- \delta)}.
\end{align}
The step-$(ii)$ follows since $|\cn_{row}(i)| \leq \tau n \eta_{row}^{d_1/ \lambda}$ (because of sub-sampling the neighboring rows). The inequality $(iii)$ follows as there exists a constant $\tau' > 0$ such that, 
\begin{align}
    \frac{\tau n \eta_{row}^{d_1/\lambda}}{\parenth{1-\delta}n\parenth{\frac{\eta_{row}^2-\frac{\Delta_r}{\sqrt{m}}}{L^2}}^{\frac{d_1}{2\lambda}}} \leq \tau' \mbox{ for all } n.
\end{align}
Combining all the computations above we get that the MSE of the two-sided nearest neighbor is bounded above by the following with probability at least $1 - 7 \delta$, 
\begin{align}
    \mbox{MSE} &\leq 2  \mathbb{V} + 2 \mathbb{B}_1 + 2 \mathbb{B}_2 \\
    &\leq \frac{4}{g(\delta)}\left[\frac{\sigma^2 \log(2/\delta)}{z_1z_2} + \parenth{\eta_{row}^2+\frac{\Delta_r}{\sqrt{\Bar{\mbi{p}}_{i,i'}m}}} + \left(\eta_{col}^2 + \frac{2\Delta_c}{\sqrt{\Bar{\mbi{p}}_{j,j'}n}} \right)\frac{\tau'}{( 1- \delta)} \right] .
\end{align}
It can be easily shown that this is equivalent to the statement made in \Cref{thm:general result}. Moreover we observe that, 
\begin{align}
    \mbox{MSE} &\leq \frac{4}{g(\delta)}\left[\frac{\sigma^2 \log(2/\delta)}{z_1z_2} + \parenth{\eta_{row}^2+\frac{\Delta_r}{\sqrt{\Bar{\mbi{p}}_{i,i'}m}}} + \left(\eta_{col}^2 + \frac{2\Delta_c}{\sqrt{\Bar{\mbi{p}}_{j,j'}n}} \right)\frac{\tau'}{( 1- \delta)} \right] \\
    &\leq C \left[ \frac{1}{n\left(\eta_{row}^2 - \frac{\Delta_r}{\sqrt{m}} \right)^{\frac{d_1}{2\lambda}} m \left(\eta_{col}^2 - \frac{\Delta_c}{\sqrt{n}} \right)^{\frac{d_2}{2\lambda}}}+ \eta_{row}^2 + \frac{\Delta_r}{\sqrt{m}} + \eta_{col}^2 + \frac{\Delta_c}{\sqrt{n}} \right].
\end{align}
Similar to the proof of \Cref{cor:main_result} we can show that for $n = \omega(\max\{m^{\frac{d_1}{2\lambda + d_2}},  m^{\frac{4\lambda}{d_1 + d_2 - 2\lambda}}\})$ and $n = O(\min\{m^{\frac{2\lambda + d_1}{d_2}},  m^{\frac{d_1 + d_2 - 2\lambda}{4\lambda}}\})$ the MSE of the two-sided nearest neighbor algorithm with $\eta_{row}=\eta_{col}=\Theta\parenth{(mn)^\frac{-\lambda}{2\lambda+d_1+d_2}}$ achieves the non-parametric minimax optimal rate $O((mn)^{\frac{-2\lambda}{2\lambda + d_1 + d_2}})$.

\section{Proof of \Cref{thm:point_wise_guarantees}}
\label{sec:pointwise_bounds}
In this section we derive the pointwise bounds for the estimates $\widehat \theta_{i, j}$. Like in the earlier proofs, we start with a bias-variance decomposition of the pointwise errors, 
\begin{align}
\label{eq:mse_decomp_ij}
    (\widehat \theta_{i, j} - f(i , j))^2 &= \sum_{i' \in \cn_{row}(i), j' \in \cn_{col}(j)} \frac{1}{|\deno|^2}\left((\theta_{i', j'} + \epsilon_{i',j'}) - \theta_{i , j} \right)^2 \\
    &\leq 2 \left(\frac{\sum_{i' \in \cn_{row}(i), j' \in \cn_{col}(j)} (\theta_{i', j'} - \theta_{i , j})}{|\deno|} \right)^2 + 2 \left( \frac{\sum_{i' \in \cn_{row}(i), j' \in \cn_{col}(j)} \epsilon_{i',j'}}{|\deno|} \right)^2 \\
    & = 2 \times \mathbb{B}_{i, j} +  2 \times \mathbb{V}_{i,j}. 
\end{align}
We know from prior discussions that under \Cref{assump:min_row_col} we can get an upper bound on the variance term using Hoeffding's inequality [\cite{bentkus2004hoeffding}] on the noise terms $\epsilon_{i',j'}$'s, 
\begin{align}
\label{eq:variance_bnd_ij}
  \mathbb{P} \left(\left|\frac{\sum_{i' \in \cn_{row}(i), j' \in \cn_{col}(j)} \epsilon_{i',j'}}{|\deno|} \right| > \zeta \Big| \mathcal{U}, \mathcal{V}\right) = \mathbb{P}\left(|\widehat \theta_{i, j} - \Tilde{\theta}_{i, j}| > \zeta | \mathcal{U}, \mathcal{V} \right) \leq 2 \exp \left\{- \frac{\zeta^2 g(\delta) z_1 z_2}{2 \sigma^2} \right\}.
\end{align}
Thus we get the following finite-sample guarantee on $\mathbb{V}_{i,j}$,
\begin{align}
    \mathbb{P}\parenth{\mathbb{V}_{i,j} \leq  \frac{2 \sigma^2 \log(2/\delta)}{g(\delta) z_1 z_2} \Big| \mathcal{U}, \mathcal{V}} &= \mathbb{P}\parenth{\left( \frac{\sum_{i' \in \cn_{row}(i), j' \in \cn_{col}(j)} \epsilon_{i',j'}}{|\deno|} \right)^2  \leq  \frac{2 \sigma^2 \log(2/\delta)}{g(\delta) z_1 z_2} \Big| \mathcal{U}, \mathcal{V}} \stackrel{\eqref{eq:variance_bnd_ij}}{\geq} 1- 2\delta. 
\end{align}

We now state and prove a lemma that would help us in getting a bound on the bias term. 
\begin{lemma}
    \label{lem:chaper_lemma}
Consider a function $g:[0,1]^d \mapsto R$ that is $(\lambda, L)$ holder continuous (w.r.t.\ $||\cdot||_{\infty}$). Let $X \sim \mathcal{P}$ where $\mathcal{P}$ is a distribution on $[0, 1]^d$ with density lower bounded by $c_{\mathcal{P}}$. Suppose $\mu_2$ denotes the second moment of $X$. Then we have the following as $\mu_2 \rightarrow 0$, 
\begin{align}
    ||g||_{\infty} = O_d \left(L^{\frac{d}{d + 2\lambda}} \left(\frac{\mu_2}{c_{\mathcal{P}}} \right)^{\frac{\lambda}{d + 2 \lambda}}\right). 
\end{align}
\end{lemma}
\begin{proof}[Proof of \Cref{lem:chaper_lemma}]
Let $||g||_{\infty} = B$ and let $x^* \in [0,1]^d$ such that $g(x^*) = B$. By the holder-continuity property of the function $g$ we can make the following derivations, 
\begin{align}
    & g(x) \geq g(x^*) - L||x - x^*||_{\infty}^{\lambda} \\
    \implies & \mu_2 \geq c_{\mathcal{P}} \int_{[0, 1]^d} \left[\left( B - L||x - x^*||_{\infty}^{\lambda} \right)_+ \right]^2 dx \\
    \implies & \mu_2 \geq c_{\mathcal{P}} \int_{[0, 1]^d} \left[\left( B - L||x ||_{\infty}^{\lambda} \right)_+ \right]^2 dx .
\end{align}
The last line follows by imitating the proof of $(116)$ in the proof of Lemma $H.2$ in \cite{dwivedi2022counterfactual}. Using the symmetry of $||x||_{\infty}$ yields, 
\begin{align}
  \int_{[0, 1]^d} \left[\left( B - L||x ||_{\infty}^{\lambda} \right)_+ \right]^2 dx &= d  \int_{\{x_1 \geq x_2, \cdots, x_d\} \cap [0,1]^d } \left[\left( B - L|x_1 |^{\lambda} \right)_+ \right]^2 dx  \\
  &= d \int_{0}^1 \left[\left( B - L|x_1 |^{\lambda} \right)_+ \right]^2 x_1^{d - 1} dx_1 \\
  &= d \int_{0}^{\min\{ 1, (B/L)^{1/\lambda}\}} \left(B - Lx_1^{\lambda} \right)^2 x_1^{d - 1} dx_1 \\
  &= d \left[\frac{B^2 x_1^d}{d} + \frac{L^2 x_1^{2\lambda + d}}{2 \lambda + d} - \frac{2 BL x_1^{\lambda + d}}{\lambda + d} \right]_0^{\min\{ 1, (B/L)^{1/\lambda}\}} \\
  &= d \min \left\{\frac{B^2}{d} + \frac{L^2}{d + 2 \lambda} - \frac{2BL}{d + \lambda}, \frac{B^{2 + \frac{d}{\lambda}}}{L^{d/\lambda}} \left[ \frac{1}{d} + \frac{1}{d + 2\lambda} - \frac{2}{d+ \lambda}\right] \right\}. 
\end{align}
Combining the pieces we have for suitable constant $c$, 
\begin{align}
    \mu_2 \geq c_{\mathcal{P}}d \min \left\{c(B^2 + L^2), cB^{2 + (d/\lambda)}/L^{d/\lambda} \right\}. 
\end{align}
Since $L$ is fixed as $\mu_2 \rightarrow 0$ we have the following, 
\begin{align}
    B = O_d \left(L^{\frac{d}{d + 2\lambda}} \left(\frac{\mu_2}{c_{\mathcal{P}}} \right)^{\frac{\lambda}{d + 2 \lambda}}\right). 
\end{align}
This completes the proof of the lemma. 
\end{proof}
We now return to the step of obtaining a bound on the bias term. Let $g^j_{row, i'}$ denote the $(\lambda, L)$ holder continuous map $v_j \mapsto f(u_{i'}, v_j) - f(u_i, v_j)$ and $g^i_{col, j'}$ denote the $(\lambda, L)$ holder continuous map $u_i \mapsto f(u_{i}, v_{j'}) - f(u_i, v_j)$. We observe that, 
\begin{align}
\label{eq:bias_decomp_ij}
    \mathbb{B}_{i,j} \stackrel{\eqref{eq:mse_decomp_ij}}{=}&  \left(\frac{\sum_{i' \in \cn_{row}(i), j' \in \cn_{col}(j)} (\theta_{i', j'} - \theta_{i , j})}{|\deno|} \right)^2 \\
   \stackrel{(i)}{\leq} & 2 \left(\frac{\sum_{i' \in \cn_{row}(i), j' \in \cn_{col}(j)} (\theta_{i', j'} - \theta_{i' , j})}{|\deno|} \right)^2 + 2 \left(\frac{\sum_{i' \in \cn_{row}(i), j' \in \cn_{col}(j)} (\theta_{i', j} - \theta_{i , j})}{|\deno|} \right)^2 \\
   \stackrel{(ii)}{\leq} & 2 \max_{j' \in \cn_{col}(j)} ||g^{i'}_{col, j'}||_{\infty}^2 + 2 \max_{i' \in \cn_{row}(i)} ||g^{j}_{row, i'}||_{\infty}^2 .
\end{align}
The step-$(i)$ follows from the basic inequality $(a+b)^2 \leq 2(a^2 + b^2)$. The step-$(ii)$ is a consequence of the definitions of the functions $g^{i'}_{col, j'}$ and $g^{j}_{row, i'}$. As in the proof of \Cref{lem:nn_conc_lemma} we define $\rho^*_{i, i'} = \mathbb{E}_{v}[(f(u_i, v) - f(u_{i'}, v))^2| \mathcal{U}]$ and $\rho^*_{j, j'} = \mathbb{E}_{u}[(f(u, v_j) - f(u, v_{j'}))^2| \mathcal{V}]$. From $\Cref{asump_row_col}$ we know that: $(i)$ the row latent factors $u_i$'s are drawn iid from uniform distribution with support $[0,1]^{d_1}$ which has a density lower bounded by a constant $c_{\mathcal{P}, u} > 0$ (say), $(ii)$ the column latent factors $v_i$'s are drawn iid from uniform distribution with support $[0,1]^{d_2}$ which has a density lower bounded by a constant $c_{\mathcal{P}, v} > 0$ (say). We bound $\mathbb{B}_{i,j}$ as follows. 
\begin{align}
\label{eq:bias2_decomp_ij}
 &\mathbb B_{i,j}  \\
 \stackrel{\eqref{eq:bias_decomp_ij}}{\leq} & 2 \max_{j' \in \cn_{col}(j)} ||g^{i'}_{col, j'}||_{\infty}^2 + 2 \max_{i' \in \cn_{row}(i)} ||g^{j}_{row, i'}||_{\infty}^2
 \stackrel{(L\ref{lem:chaper_lemma})}{\leq} & 2 \max_{j' \in \cn_{col}(j)} c_{d_1} L^{\frac{2d_1}{d_1 + 2 \lambda}} \left(\frac{\rho^*_{j, j'}}{c_{\mathcal{P}, u}} \right)^{\frac{2\lambda}{d_1 + 2 \lambda}} + 2 \max_{i' \in \cn_{row}(i)} c_{d_2} L^{\frac{2d_2}{d_2 + 2 \lambda}} \left(\frac{\rho^*_{i, i'}}{c_{\mathcal{P}, v}} \right)^{\frac{2\lambda}{d_2 + 2 \lambda}} \\
 \stackrel{(L\ref{lem:nn_conc_lemma})}{\leq} & 2 \max_{j' \in \cn_{col}(j)} c_{d_1} L^{\frac{2d_1}{d_1 + 2 \lambda}} \left(\frac{\widehat d^2(j, j') + \Delta_c/\sqrt{\Bar{\mbi{p}}_{j,j'}n}}{c_{\mathcal{P}, u}} \right)^{\frac{2\lambda}{d_1 + 2 \lambda}} + 2 \max_{i' \in \cn_{row}(i)} c_{d_2} L^{\frac{2d_2}{d_2 + 2 \lambda}} \left(\frac{\widehat d^2(i,i') + \Delta_r/\sqrt{\Bar{\mbi{p}}_{i,i'}m}}{c_{\mathcal{P}, v}} \right)^{\frac{2\lambda}{d_2 + 2 \lambda}} \\
 \stackrel{\eqref{eq:defn_nrowcol}}{\leq} & 2c_{d_1}' L^{\frac{2d_1}{d_1 + 2 \lambda}} \left(\frac{\eta_{col}^2 + \Delta_c/\sqrt{n}}{c_{\mathcal{P}, u}} \right)^{\frac{2\lambda}{d_1 + 2 \lambda}} + 2 c_{d_2}' L^{\frac{2d_2}{d_2 + 2 \lambda}} \left(\frac{\eta_{row}^2 + \Delta_r/\sqrt{m}}{c_{\mathcal{P}, v}} \right)^{\frac{2\lambda}{d_2 + 2 \lambda}}. 
\end{align}
In the above $c_{d_1}, c_{d_1}', c_{d_2}, c_{d_2}'$ are suitable constants. Thus we conclude that with probability at-least $ 1- 4\delta$ the following pointwise error bounds hold, 
\begin{align}
    (\widehat \theta_{i, j} - f(i , j))^2 
    \stackrel{\eqref{eq:mse_decomp_ij}}{\leq} & 2 \times \mathbb{B}_{i, j} +  2 \times \mathbb{V}_{i,j} \\
     \stackrel{\eqref{eq:bias2_decomp_ij}}{\leq} & 2c_{d_1}' L^{\frac{2d_1}{d_1 + 2 \lambda}} \left(\frac{\eta_{col}^2 + \Delta_c/\sqrt{n}}{c_{\mathcal{P}, u}} \right)^{\frac{2\lambda}{d_1 + 2 \lambda}} + 2 c_{d_2}' L^{\frac{2d_2}{d_2 + 2 \lambda}} \left(\frac{\eta_{row}^2 + \Delta_r/\sqrt{m}}{c_{\mathcal{P}, v}} \right)^{\frac{2\lambda}{d_2 + 2 \lambda}} +\frac{4 \sigma^2 \log(2/\delta)}{g(\delta) z_1 z_2} \\
    \stackrel{(L\ref{lem:bound_nn})}{\leq} & 2c_{d_1}' L^{\frac{2d_1}{d_1 + 2 \lambda}} \left(\frac{\eta_{col}^2 + \Delta_c/\sqrt{n}}{c_{\mathcal{P}, u}} \right)^{\frac{2\lambda}{d_1 + 2 \lambda}} + 2 c_{d_2}' L^{\frac{2d_2}{d_2 + 2 \lambda}} \left(\frac{\eta_{row}^2 + \Delta_r/\sqrt{m}}{c_{\mathcal{P}, v}} \right)^{\frac{2\lambda}{d_2 + 2 \lambda}} \\
    & + \frac{4 \sigma^2 \log(2/\delta)}{g(\delta) \parenth{1-\delta}^2 nm \parenth{\frac{\eta_{row}^2-\frac{\Delta_r}{\sqrt{m}}}{L^2}}^{\frac{d_1}{2\lambda}} \parenth{\frac{\eta_{col}^2-\frac{\Delta_c}{\sqrt{n}}}{L^2}}^{\frac{d_2}{2\lambda}}}.
\end{align}
We can minimize the above upper bound to obtain optimum choice of $\eta_{row}, \eta_{col}$. It can be checked that the optimal pointwise error rate comes out to be $O\left((mn)^{\frac{-2 \lambda}{(2\lambda + d_1 + d_2) + (d_1d_2/ \lambda)}}\right)$ for suitably chosen $\eta_{row}, \eta_{col}$. 

\section{Proof of \Cref{thm: TSNN CLT}}
\label{sec:clt_proof}
We have the standard decomposition for $\widehat \theta_{i, j} - f(i, j)$, 
\begin{align}
\label{eq:mse_Decomp_clt}
    \widehat \theta_{i, j} - f(i, j) &= \left(\frac{\sum_{i' \in \cn_{row}(i), j' \in \cn_{col}(j)} (\theta_{i', j'} - \theta_{i , j})}{|\deno|} \right) + \left( \frac{\sum_{i' \in \cn_{row}(i), j' \in \cn_{col}(j)} \epsilon_{i',j'}}{|\deno|} \right) \\
    &= \mathbb{B} + \mathbb{V}. 
\end{align}
Let us first analyse the bias term $\mathbb{B}$. We shall use the results derived in \Cref{sec:pointwise_bounds}. We note that, 
\begin{align}
\label{eq:bias_analysis_clt}
    |\mathbb{B}| \stackrel{\eqref{eq:mse_Decomp_clt}}{=}& \left| \frac{\sum_{i' \in \cn_{row}(i), j' \in \cn_{col}(j)} (\theta_{i', j'} - \theta_{i , j})}{|\deno|}  \right| \\
    \stackrel{(i)}{\leq} & \left|\frac{\sum_{i' \in \cn_{row}(i), j' \in \cn_{col}(j)} (\theta_{i', j'} - \theta_{i' , j})}{|\deno|}  \right| + \left|\frac{\sum_{i' \in \cn_{row}(i), j' \in \cn_{col}(j)} (\theta_{i', j} - \theta_{i , j})}{|\deno|}  \right| \\
    \leq &  \max_{j' \in \cn_{col}(j)} ||g^{i'}_{col, j'}||_{\infty} +  \max_{i' \in \cn_{row}(i)} ||g^{j}_{row, i'}||_{\infty} \\
    \stackrel{(L\ref{lem:chaper_lemma})}{\leq} &  \max_{j' \in \cn_{col}(j)} c_{d_1} L^{\frac{d_1}{d_1 + 2 \lambda}} \left(\frac{\rho^*_{j, j'}}{c_{\mathcal{P}, u}} \right)^{\frac{\lambda}{d_1 + 2 \lambda}} +  \max_{i' \in \cn_{row}(i)} c_{d_2} L^{\frac{d_2}{d_2 + 2 \lambda}} \left(\frac{\rho^*_{i, i'}}{c_{\mathcal{P}, v}} \right)^{\frac{\lambda}{d_2 + 2 \lambda}} \\
 \stackrel{(L\ref{lem:nn_conc_lemma})}{\leq} &  \max_{j' \in \cn_{col}(j)} c_{d_1} L^{\frac{d_1}{d_1 + 2 \lambda}} \left(\frac{\widehat d^2(j, j') + \Delta_c/\sqrt{\Bar{\mbi{p}}_{j,j'}n}}{c_{\mathcal{P}, u}} \right)^{\frac{\lambda}{d_1 + 2 \lambda}} + \max_{i' \in \cn_{row}(i)} c_{d_2} L^{\frac{d_2}{d_2 + 2 \lambda}} \left(\frac{\widehat d^2(i,i') + \Delta_r/\sqrt{\Bar{\mbi{p}}_{i,i'}m}}{c_{\mathcal{P}, v}} \right)^{\frac{\lambda}{d_2 + 2 \lambda}} \\
 \stackrel{\eqref{eq:defn_nrowcol}}{\leq} & c_{d_1}' L^{\frac{d_1}{d_1 + 2 \lambda}} \left(\frac{\eta_{col}^2 + \Delta_c/\sqrt{n}}{c_{\mathcal{P}, u}} \right)^{\frac{\lambda}{d_1 + 2 \lambda}} +  c_{d_2}' L^{\frac{d_2}{d_2 + 2 \lambda}} \left(\frac{\eta_{row}^2 + \Delta_r/\sqrt{m}}{c_{\mathcal{P}, v}} \right)^{\frac{\lambda}{d_2 + 2 \lambda}}. 
\end{align}
Here step-$(i)$ follows because of the fact that $|a + b| \leq |a | + |b|$. Hence we observe that, 
\begin{align}
\label{eq:bias_clt}
    \sqrt{|\deno|} |\mathbb{B}| &\stackrel{\eqref{eq:bias_analysis_clt}}{=} O_p\left(\sqrt{|\deno|} \left\{\eta_{row}^{\frac{2\lambda}{d_2 + 2 \lambda}} + \eta_{col}^{\frac{2\lambda}{d_1 + 2 \lambda}} \right\} \right) \\
   &\leq O_p\left(\sqrt{ T_{n,m}\left\{\eta_{row}^{\frac{4\lambda}{d_2 + 2 \lambda}} + \eta_{col}^{\frac{4\lambda}{d_1 + 2 \lambda}} \right\}} \right) \\
   &\stackrel{(A\ref{assump:subsample})}{=} O_p(o_P(1)) \\
   &\stackrel{(ii)}{=} o_P(1). 
\end{align}
The step-$(ii)$ uses the property that if we have two sequences of random variables $\{X_n\}_{n = 1}^{\infty}$ and $\{Y_n\}_{n = 1}^{\infty}$ such that $X_n \stackrel{P}{\rightarrow} 0$ and $Y_n = O_p(X_n)$, then $Y_n \stackrel{P}{\rightarrow} 0$ as well. By standard central limit theorem on the iid error terms $\epsilon_{i',j'}$ we have the following distributional convergence given $\mathcal{U, V}$. 
\begin{align}
\label{eq:var_clt}
    \sqrt{|\deno|} |\mathbb{V}| &\stackrel{\eqref{eq:mse_Decomp_clt}}{=} \sqrt{|\deno|} \frac{\sum_{i' \in \cn_{row}(i), j' \in \cn_{col}(j)} \epsilon_{i',j'}}{|\deno|} \\
    &= \frac{\sum_{i' \in \cn_{row}(i), j' \in \cn_{col}(j)} \epsilon_{i',j'}}{\sqrt{|\deno|}} \\
    &\stackrel{d}{\rightarrow} \mathcal{N}(0, \sigma^2). 
\end{align}
We obtain the distributional convergence of $\widehat \theta_{i,j}$ by combining all the above results, 
\begin{align}
\label{eq: TSNN CLT}
    \sqrt{\abss{\deno}}  \parenth{\what{\theta}_{i, j} - \theta_{i, j}} \stackrel{\eqref{eq:mse_Decomp_clt}}{=} &  \sqrt{|\deno|}  \left( \mathbb{B} + \mathbb{V} \right) \\
     \xrightarrow[\eqref{eq:bias_clt}, \eqref{eq:var_clt}]{d} & \mathcal{N}(0, \sigma^2) + o_P(1) \\
     \underset{(iii)}{\stackrel{d}{=}} &   \mathcal{N}(0, \sigma^2). 
\end{align}
In the above computation step-$(iii)$ follows from Slutsky's theorem.

\section{Deferred simulation details}
\label{sec: appendix D}
\subsection{Deffered details about Simulation experiments}
\label{subsec: simulation appendix}
\paragraph{Tuning $\eta$ and reporting test error} We do 5 fold cross-validation to tune the $\eta$ in TS-NN(other nearest neighbors are implemented in similar fashion) and report its test error. At first, matrix entries are arbitrarily assigned to 5 different folds. One of the folds is held out as the test data and the other 4 folds are used for training the NNs, denote them as $\testfold$ and $\mathcal{F}_{train}$ respectively.

We calculate the row-wise and column-wise distances $\braces{\what{d}^2_{row}(i,i')}_{i,i'\in [n]}$ and $\braces{\what{d}^2_{col}(j,j')}_{j,j'\in [m]}$ from the training dataset. For practical purposes, we use the following definitions of distances
\begin{align}
\label{eq: computation distances}
    \widehat d_{row}^2(i , i') &=  \frac{ \sum_{(i,j),(i',j)\in \trainfold} (X_{i,j} - X_{i', j})^2 A_{i, j}A_{i',j}}{\sum_{(i,j),(i',j)\in \trainfold} A_{i, j}A_{i', j}},\\
    \widehat d_{col}^2(j, j') &=  \frac{ \sum_{(i,j),(i',j)\in \trainfold} (X_{i,j} - X_{i,j'})^2 A_{i, j}A_{i, j'}}{\sum_{(i,j),(i',j)\in \trainfold} A_{i, j}A_{i, j'}}.
\end{align}
Now let $\what{\theta}_{i,j,\mbi{\eta}}$ be the TS-NN($\mbi{\eta}$)'s estimate of $\theta_{i,j}$ for a specified threshold $\mbi{\eta} = \parenth{\eta_{row}, \eta_{col}}$, using the calculated distances. We compute the grid of $t$ $\eta_{row}$'s and $\eta_{col}$'s using certain quantiles of $\braces{\what{d}^2_{row}(i,i')}_{i,i'\in [n]}$ and $\braces{\what{d}^2_{col}(j,j')}_{j,j'\in [m]}$ respectively. We denote it as the $\eta_{grid,row}:=\braces{\eta_{1,row},\dots,\eta_{t,row}}$ and $\eta_{grid,col}:=\braces{\eta_{1,col},\dots,\eta_{t,col}}$. Since percentiles of the distances are unaffected by the addition of the same term $2\sigma^2$ to all the distances, we don't calculate $\what{\sigma}^2$ for $\braces{\what{d}^2_{row}(i,i')}_{i,i'\in [n]}$ and $\braces{\what{d}^2_{col}(j,j')}_{j,j'\in [m]}$ and work with \cref{eq: computation distances}. Then we tune $\eta$ in the training folds as follows:
\begin{align}
    \eta_{tuned} = \arg\min_{\mbi{\eta}\in\eta_{grid,row}\times\eta_{grid,col} } \frac{\sum_{(i,j)\in\trainfold}\parenth{Y_{i,j} - \what{\theta}_{i,j,\mbi{\eta}}}^2A_{i,j}}{\sum_{(i,j)\in\trainfold}A_{i,j}}.
\end{align}
 Then the test error is calculated as the mean squared error on the test fold, denoted as $\testfold$
 \begin{align}
     \what{\sigma}^2_{test} = \frac{\sum_{(i,j)\in\testfold}\parenth{Y_{i,j} - \what{\theta}_{i,j,\mbi{\eta}_{tuned}}}^2A_{i,j}}{\sum_{(i,j)\in\testfold}A_{i,j}}.
 \end{align}
We repeat this process 5 times, each time assigning a different fold as the test fold and report the average of the $\what{\sigma}^2_{test}$'s as the final test error in \cref{fig: merged comparison} and \cref{fig: lambda = 0.5 benchmark comparison}.

For the $\eta_{grid,row},\eta_{grid,col}$, we work with the percentiles ranging from 1.5 to 10(same is true for DR-NN). For one-sided NNs, we expand the percentiles range to 1.5 - 30 to make them slightly more powerful and further highlight the importance of combining row and column neighbors for matrix estimation.

\paragraph{Fitting SoftImpute} We fit SoftImpute using the R package \texttt{softImpute} [\cite{softimpute}]. SoftImpute uses nuclear norm regularization and we fit SoftImpute with $\lambda$ varying over a log grid from 1 to 12. Then we report the minimum MSE among all the MSEs obtained via SoftImpute for various $\lambda$'s. We choose this grid as we found that the optimum lambda is almost always lied in the interior of this grid.

\subsubsection{Estimating $\sige$ for getting confidence intervals} 
\label{subsubsec: clt appendix}
We use the 5-fold data split in the coverage experiments (\cref{fig: merged clt comparison}). After training TS-NN, we set the estimate of noise variance $\sige$ as follows:
\begin{align}
\label{eq: err_sd_estimate}
    \what{\sig}_{\eps}:=\frac{\sum_{(i,j)\in\trainfold}\parenth{Y_{i,j} - \what{\theta}_{i,j,\mbi{\eta}_{tuned}}}^2A_{i,j}}{\sum_{(i,j)\in\trainfold}A_{i,j}}.
\end{align}
$\what{\sig}_{\eps}$ is a consistent estimator for noise SD $\sige$ (See \cref{rem: consistency of noise sd}). Empirically, \cref{fig: merged clt comparison} verifies the consistency of $\what{\sig}_{\eps}$.

\begin{remark}[finite sample adjustment]
    \label{rem: finite sample adjust}
    Now, for the following $\parenth{1-\alpha}$ CIs
\begin{align}
\label{eq: unboosted ci}
    \parenth{\what{\theta}_{i,j} -\frac{z_{\alpha/2}\what{\sigma}_{\eps} }{\sqrt{\abss{\deno}}} , \what{\theta}_{i,j} +\frac{z_{\alpha/2}\what{\sigma}_{\eps} }{\sqrt{\abss{\deno}}}};\quad \quad \parenth{\what{\theta}_{i,j} -\frac{z_{\alpha/2}\sige }{\sqrt{\abss{\deno}}} , \what{\theta}_{i,j} +\frac{z_{\alpha/2}\sige }{\sqrt{\abss{\deno}}}} 
\end{align}

coverage rate is coming out to be $\sim 70\%$ even for matrices with $n=m=300$. So, to boost finite sample coverage, we add the within-nearest-neighbors SD to the consistent estimate of noise SD i.e., we work with the following CI

\begin{align}
    \parenth{1-\alpha}\% \quad\mbi{\what{CI}}:=\parenth{\what{\theta}_{i,j} -\frac{z_{\alpha/2}\tilde{\sigma}_{i,j,\eps} }{\sqrt{\abss{\deno}}} , \what{\theta}_{i,j} +\frac{z_{\alpha/2}\tilde{\sigma}_{i,j,\eps} }{\sqrt{\abss{\deno}}}}\qtext{where} \tilde{\sigma}_{i,j,\eps} = \what{\sig}_\eps + \what\sig_{i,j}
\end{align}
where $\what\sig^2_{i,j}=\frac{\sum_{(i',j')\in\deno}\parenth{Y_{i',j'} - \what{\theta}_{i,j,\mbi{\eta}_{tuned}}}^2A_{i',j'}}{\abss{\deno} - 1}$ if $\abss{\deno} >1$ else $\what\sig^2_{i,j} = 0$ and $\what{\sig}_\eps$ is given by \cref{eq: err_sd_estimate}. In similar fashion, oracular CIs with finite sample adjustment is defined as
\begin{align}
    \parenth{1-\alpha}\% \quad\mbi{CI_o}:=\parenth{\what{\theta}_{i,j} -\frac{z_{\alpha/2}\sig^\ddagger_{i,j,\eps} }{\sqrt{\abss{\deno}}} , \what{\theta}_{i,j} +\frac{z_{\alpha/2}\sig^\ddagger_{i,j,\eps} }{\sqrt{\abss{\deno}}}}\qtext{where} \sig^\ddagger_{i,j,\eps} = \sig_\eps + \what\sig_{i,j}
\end{align}
\end{remark}

For evaluating the coverage in one data split $\trainfold$-$\testfold$, we train TS-NN on $\trainfold$ and then look at the average proportion of counterfactuals $\theta_{i,j}$ in $\testfold$ covered by $(1-\alpha)\%$ $\mbi{\what{CI}}$ and $(1-\alpha)\%$ $\mbi{CI_o}$. We repeat this process 5 times, in each iteration we assign a different fold to the $\testfold$. Ultimately, we report the average of 5 $\testfold$ coverages as the final empirical coverage rate of $(1-\alpha)\%$ $\mbi{\what{CI}}$ and $(1-\alpha)\%$ $\mbi{CI_o}$.


\subsubsection{Additional Details}
\paragraph{Runtime Complexity of NN based methods}
In terms of runtime complexity for a $n \times m$ matrix, Row - NN has $\mathcal{O}({n\choose2}m + n) = \mathcal{O}(n^2m +n)$. First term arises as there are ${n\choose 2}$ combinations of rows and computing $L_2$ distance between each pair of rows take $\mathcal{O}(m)$ time. The second term arises due to aggregation of $\mathcal{O}(n)$ terms. Similarly Col - NN has a runtime complexity of $\mathcal{O}(m^2n +m)$. Finally, TS-NN has a runtime complexity of $\mathcal{O}(m^2n +n^2m + mn)$. DR-NN can be expressed as a linear combination of Row-NN, Col-NN and TS-NN~\citep{dwivedi2022doubly}, hence it also has a runtime complexity of $\mathcal{O}(m^2n +n^2m + mn)$.
\paragraph{Row-NN algorithm}
\label{subsec: additional exp details}
Just for clarity, the Row-NN algorithm is outlined below: 
\begin{itemize}
    \item[Step-1:] Compute the pairwise row distance estimates $\widehat d^2_{row}(i , j)$ for all $i, j \in [n]$ and use it to construct the neighborhood of row $i$, 
    \begin{align}
         \cn_{row}(i) &=\{j \in [n]: \hat{d}_{row}^2(i, j )\leq \eta_{row}\}.
    \end{align}
    Here $\eta_{row}\geq 0$ is the tuning parameter. 
    \item[Step-2:] The estimate of $\theta_{i,t}$ is given by the row nearest neighbor estimate, 
    \begin{align}
        \widehat \theta_{i, t} = \frac{\sum_{j \in \cn_{row}(i)}X_{j,t}A_{j, t}}{|\cn_{row}(i)|}, 
    \end{align}
\end{itemize}
Col-NN has the column-counterpart algorithm of the above procedure.
\subsection{Deffered details of real-life case study: HeartSteps}
\label{subsec: heartsteps appendix}
We will now provide additional details about the HeartSteps dataset.

We observe that at each decision time point, the HeartSteps algorithm determined whether a user is available based on certain attributes like whether the user was driving a car, etc. Other features are whether the user had an active connection at or around the decision time, was not in transit and phone was not in snooze mode. We focus on the matrix completion at the ``available decision times". After screening out the ``non-avaliable times", we see that only few users have $>210$ decision times, resulting in filtering out the remaining columns/ decision times from consideration. Ultimately, we work with a dataset of 37 rows/users and 210 columns/decision times. For further details, we refer the readers to \cite{klasnja2019efficacy}.

For implementation, we use 5-fold blocked cross validation. HeartSteps' underlying recommender algorithm uses a user's history to sequentially assign interventions. So, there is a temporal dependence of the later columns on the previous columns. To tackle that, we first of all divide the the rows of the matrix into 5 folds. Now in each iteration we fix a fold, we hold out the entries in the last 40 columns of those rows as our test dataset. Remaining entries are used for training the NNs, USVT and SoftImpute. For NNs, we use the tuned $\eta_{row}$ and $\eta_{col}$(or one of them for Row-NN and Col-NN) to complete the matrix. Then we report the difference between hold-out entries and their corresponding fitted estimates. We see the estimated matrices in SoftImpute and USVT resulted in extremely high test errors compared to NN-based methods. For $\eta_{grid,row},\eta_{grid,col}$ in one-sided NNs and DR-NN, we work with the percentiles 25-85. Low percentiles were avoided as they gave the lowest training errors but the test errors were exorbitantly high. For TS-NN $\mbi{\eta}$ grid, we consider percentiles 8 - 50.

\subsection{Additional Plots}

\begin{figure*}[ht]
    \centering
    \begin{tabular}{cc}
         \quad\qquad \textbf{MCAR} & \quad \qquad \textbf{MNAR} \\
        \includegraphics[trim=0in 0.5in 0in 0in, clip, width=0.45\textwidth]{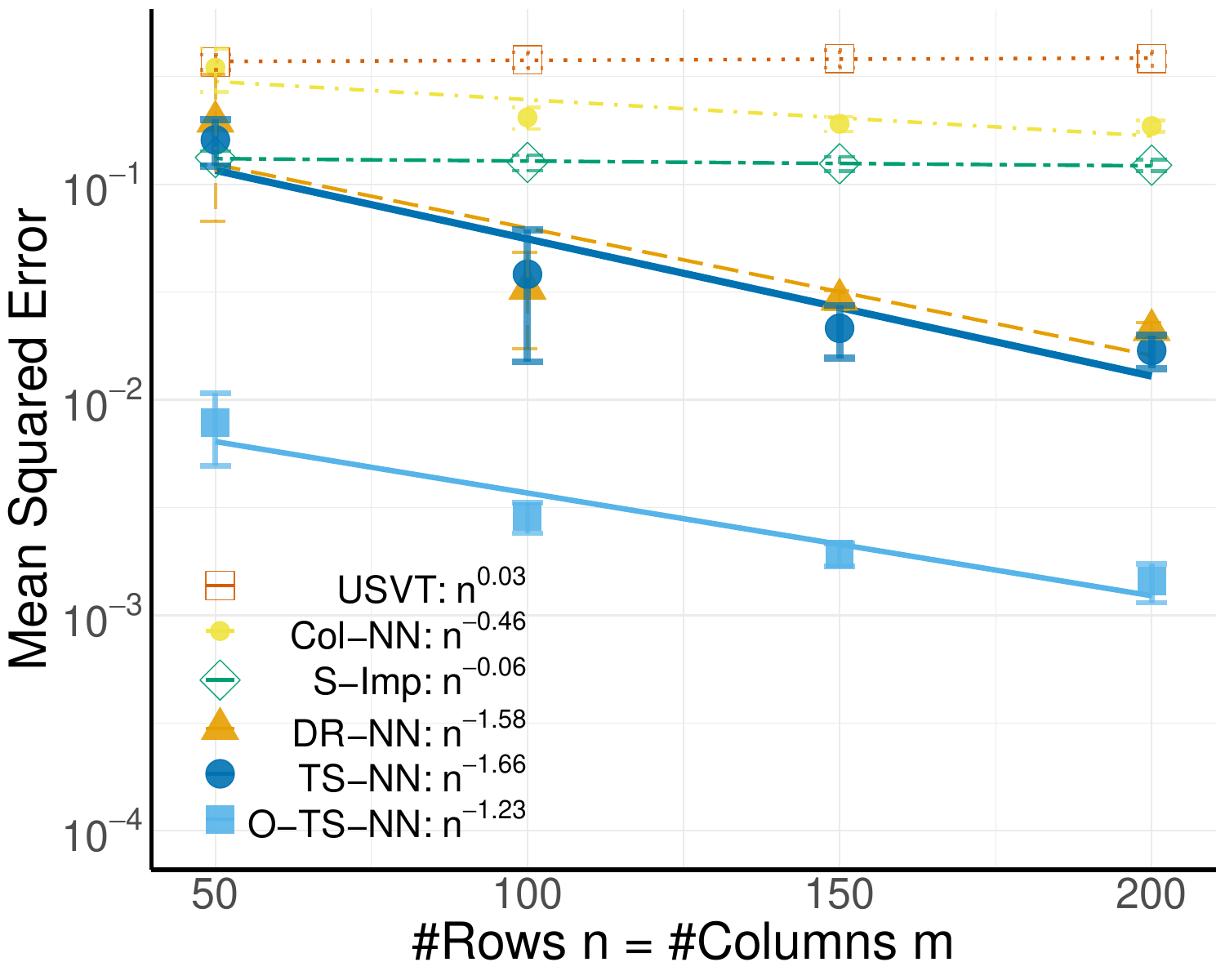}&
        \includegraphics[trim=0in 0.5in 0in 0in, clip, width=0.45\textwidth]{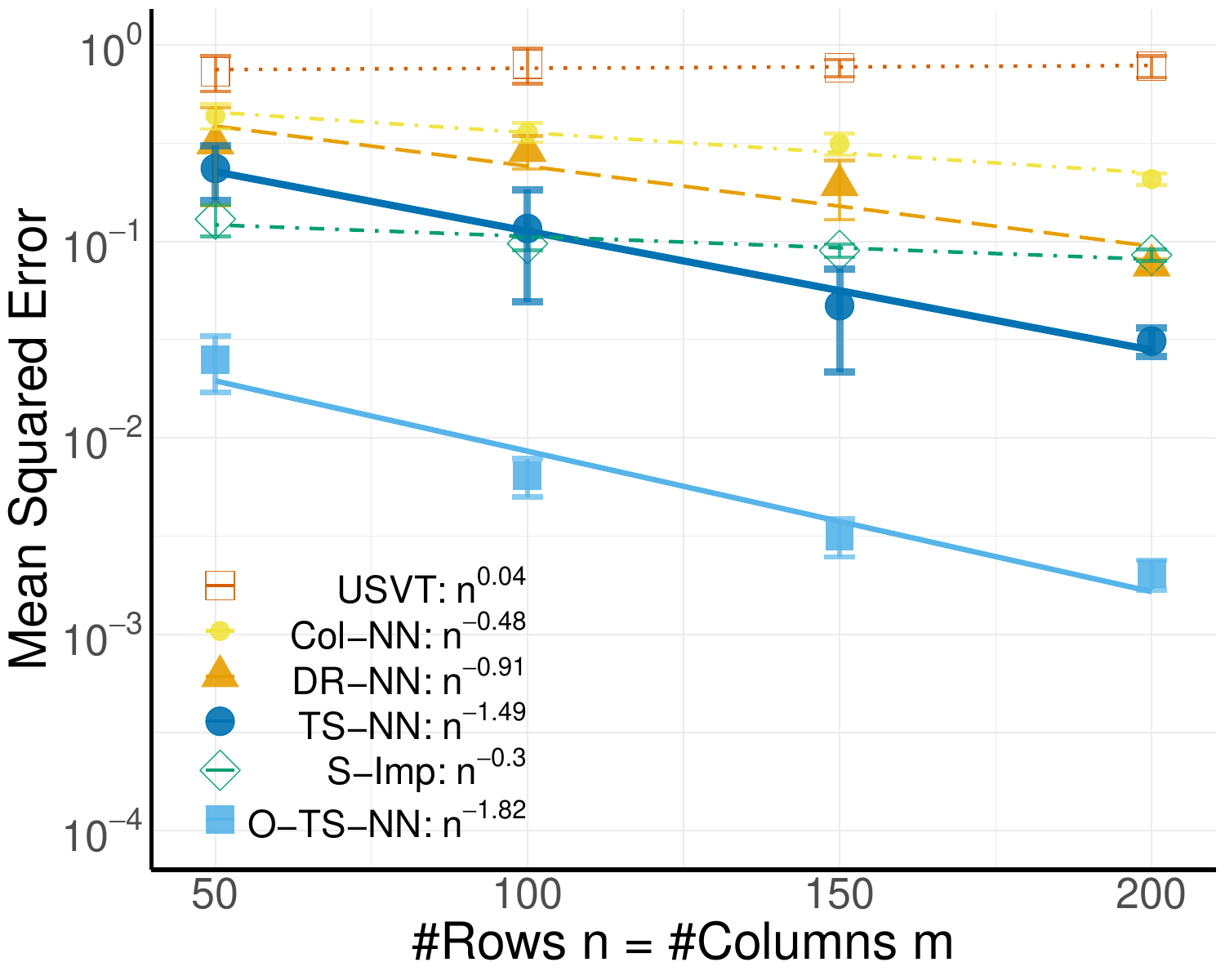}\\[-1mm]
        \ \ \quad \# Rows $n$ &\ \  \quad  \# Rows $n$
    \end{tabular}
    \\[2mm]
    \caption{\textbf{MSE of different algorithms for estimating $\theta_{i,t}$ as a function of $n$ when $\lambda=0.5$ and SNR $= 2$.} TS-NN demonstrates quantifiable improvements over USVT, SoftImpute, and other NNs in estimation both in terms of MSE value and MSE decay rate. Moreover, TS-NN shows similar (if not better) MSE decay rates with $n$ as compared to its oracle version in MCAR setup. Over here, we keep $n=m$ to keep the interpretation uncomplicated.}
    \label{fig: lambda = 0.5 benchmark comparison}
\end{figure*}

\vskip 0.2in
\bibliography{refs}

\end{document}

%% file: cref_setup.tex
\newtheorem{myproposition}{Proposition}

\newtheorem{mycorollary}{Corollary}

\newtheorem{mylemma}{Lemma}

\newtheorem{mydefinition}{Definition}

\newtheorem{myremark}{Remark}

\newtheorem{myassumption}{Assumption}

 \crefname{appendix}{App.}{App.}
\crefname{equation}{}{}
\Crefname{equation}{}{}
\crefname{lemma}{Lem.}{Lem.}
\crefname{theorem}{Thm.}{Thm.}
\crefname{Corollary}{Cor.}{Cors.}
\crefname{algorithm}{Alg.}{Algs.}

\crefname{section}{Sec.}{Sec.}
\crefname{table}{Tab.}{Tab.}
\crefname{remark}{Rem.}{Rem.}
\crefname{definition}{Def.}{Def.}
\crefname{Proposition}{Prop.}{Prop.}
\crefname{myremark}{Rem.}{Rem.}
\crefname{mylemma}{Lem.}{Lem.}
\crefname{mydefinition}{Def.}{Defs.}
\crefname{myproposition}{Prop.}{Prop.}
\crefname{mycorollary}{Cor.}{Cors.}
\crefname{myassumption}{Assum.}{Assum.}
\crefname{figure}{Fig.}{Fig.}
\crefname{enumi}{}{}
\crefname{name}{}{} 

%% file: raaz-macros.tex
\newcommand{\bias}{\mathbb{B}}
\newcommand{\trainfold}{\mathcal{F}_{train}}
\newcommand{\testfold}{\mathcal{F}_{test}}
\newcommand{\MSE}{\mathrm{MSE}}
\newcommand{\iidsim}{\stackrel{iid}{\sim}}
\newcommand{\sige}{\sigma_{\eps}}
\newcommand{\cn}{\mathcal{N}}

\newcommand{\etarow}[1][i]{\cn_{row}(#1)}
\newcommand{\etacol}[1][j]{\cn_{col}(#1)}
\newcommand{\signalij}{\theta_{i,j}}
\newcommand{\hatij}{\what{\theta}_{i,j}}
\newcommand{\deno}{\cn_{row,col}(i, j)}

\newcommand{\sumn}[1][i]{\sum_{#1=1}^n}

\newcommand{\x}{x}

\newcommand{\axi}[1][i]{\x_{#1}}

\newcommand{\eps}{\epsilon}

\newcommand{\pseqxn}[1][n]{(\axi[i])_{i\geq 1}} 
\newcommand{\pseqxnn}[1][n]{(\axi[i])_{i=1}^n} 

\newcommand{\parenth}[1]{\left( #1 \right)}

\newcommand{\braces}[1]{\left\{ #1 \right \}}

\newcommand{\abss}[1]{\left| #1 \right |}

\newcommand{\real}{\ensuremath{\mathbb{R}}}

\newcommand{\Prob}{\ensuremath{{\mathbb{P}}}}

\newcommand{\sig}{\sigma}

\def\balign#1\ealign{\begin{align}#1\end{align}}
\def\baligns#1\ealigns{\begin{align*}#1\end{align*}}
\def\balignat#1\ealign{\begin{alignat}#1\end{alignat}}
\def\balignats#1\ealigns{\begin{alignat*}#1\end{alignat*}}
\def\bitemize#1\eitemize{\begin{itemize}#1\end{itemize}}
\def\benumerate#1\eenumerate{\begin{enumerate}#1\end{enumerate}}

\newenvironment{talign*}
 {\csname align*\endcsname}
 {\endalign}
\newenvironment{talign}
 {\csname align\endcsname}
 {\endalign}

\def\balignst#1\ealignst{\begin{talign*}#1\end{talign*}}
\def\balignt#1\ealignt{\begin{talign}#1\end{talign}}

\newcommand{\qtext}[1]{\quad\text{#1}\quad}

\let\originalleft\left
\let\originalright\right
\renewcommand{\left}{\mathopen{}\mathclose\bgroup\originalleft}
\renewcommand{\right}{\aftergroup\egroup\originalright}

\def\Holder{H\"older\xspace}

\def\tinycitep*#1{{\tiny\citep*{#1}}}
\def\tinycitealt*#1{{\tiny\citealt*{#1}}}
\def\tinycite*#1{{\tiny\cite*{#1}}}
\def\smallcitep*#1{{\scriptsize\citep*{#1}}}
\def\smallcitealt*#1{{\scriptsize\citealt*{#1}}}
\def\smallcite*#1{{\scriptsize\cite*{#1}}}

\def\mbi#1{\boldsymbol{#1}} 

\def\mbb#1{\mathbb{#1}}
\def\mc#1{\mathcal{#1}}
\def\mrm#1{\mathrm{#1}}
\def\trm#1{\textrm{#1}}
\def\tbf#1{\textbf{#1}}

\def\<{\left\langle} 
\def\>{\right\rangle}

\def\implies{\quad\Longrightarrow\quad}

\def\what#1{\widehat{#1}}

\def\indic#1{\mbb{I}\left[{#1}\right]} 

\newcommand{\Gsn}{\mathcal{N}}

\newcommand{\Ber}{\textnormal{Ber}}

\newcommand{\Unif}{\textnormal{Unif}}

\newcommand{\iid}{\textrm{i.i.d.}\xspace}

\ifdefined\nonewproofenvironments\else
\ifdefined\ispres\else

\newenvironment{proof-sketch}{\noindent\textbf{Proof Sketch}
  \hspace*{1em}}{\qed\bigskip\\}
\newenvironment{proof-idea}{\noindent\textbf{Proof Idea}
  \hspace*{1em}}{\qed\bigskip\\}
\newenvironment{proof-of-lemma}[1][{}]{\noindent\textbf{Proof of Lemma {#1}}
  \hspace*{1em}}{\qed\\}
\newenvironment{proof-of-theorem}[1][{}]{\noindent\textbf{Proof of Theorem {#1}}
  \hspace*{1em}}{\qed\\}
\newenvironment{proof-attempt}{\noindent\textbf{Proof Attempt}
  \hspace*{1em}}{\qed\bigskip\\}

\fi
\fi

